\documentclass{article}

\usepackage{microtype}
\usepackage{graphicx}
\usepackage{subcaption}
\usepackage{booktabs} % for professional tables
\usepackage{multirow} % for table row spanning

\usepackage{hyperref}

\usepackage[main, final]{neurips_2025}

\usepackage{amsmath}
\usepackage{amssymb}
\usepackage{mathtools}
\usepackage{amsthm}
\usepackage{algorithm}
\usepackage{algorithmic}
\usepackage{tikz}
\usepackage{pgfplots}
\pgfplotsset{compat=1.18}
\usetikzlibrary{positioning,decorations.pathreplacing,shapes.geometric}

\usepackage[capitalize,noabbrev]{cleveref}

\theoremstyle{plain}
\newtheorem{theorem}{Theorem}[section]
\newtheorem{proposition}[theorem]{Proposition}
\newtheorem{lemma}[theorem]{Lemma}
\newtheorem{corollary}[theorem]{Corollary}
\theoremstyle{definition}

\theoremstyle{remark}
\newtheorem{remark}[theorem]{Remark}

\usepackage[textsize=tiny]{todonotes}

\title{Risk-Averse Online POMDP Planning via CVaR of the Immediate Cost with Performance Guarantees}

\author{%
  Yaacov Pariente \\
  Faculty of Mathematics \\
  Technion -- Israel Institute of Technology \\
  Haifa, Israel \\
  \texttt{yaacovp@campus.technion.ac.il} \\
  \And
  Vadim Indelman \\
  Stephen B.\ Klein Faculty \\
  of Aerospace Engineering; \\
  Faculty of Data and Decision Sciences \\
  Technion -- Israel Institute of Technology \\
  Haifa, Israel \\
  \texttt{vadim.indelman@technion.ac.il} \\
}

\begin{document}

\maketitle

\begin{abstract}
Online POMDP planners optimize the expected cumulative cost, which can mask dangerous states when the belief places significant mass on high-cost states. Existing risk-averse methods apply static or dynamic Conditional Value at Risk (CVaR) to the value function, capturing trajectory-level risk, but share two gaps: (i) by retaining the immediate cost as an expectation of a state-dependent cost over the belief, the risk \emph{within} the belief is left unaddressed; and (ii) by modifying the value function, they require new tailored algorithms rather than reusing existing expectation-based planners. We instead apply CVaR to the immediate cost over the belief at each step, directly targeting per-step uncertainty about the current state. The standard expected cumulative return is retained as the objective, so the resulting problem has a standard MDP structure: any expectation-based POMDP planner can be made risk-sensitive by changing only the cost computation. We inherit finite-time guarantees for policy evaluation and sparse sampling---with estimation error independent of the risk level---and, as our central theoretical result, prove a finite-time bound on the gap between the particle belief MDP surrogate and the original POMDP, which together yield an end-to-end guarantee from the true POMDP value to the algorithmic estimate. In the risk-neutral limit, the formulation recovers standard expectation-based planning.
\end{abstract}

\section{Introduction}

In partially observable settings, the agent does not know its true state---it maintains a belief, a probability distribution that may place significant mass on both safe and dangerous states. Standard POMDP planners \cite{silver2010monte, sunberg2018online, kearns2002sparse} optimize the \emph{expected} cumulative cost, which averages over this belief and can mask dangerous states in the tail. In safety-critical applications such as robotics, autonomous driving, and medical decision-making, this averaging can lead to unacceptable outcomes \cite{majumdar2020should}: the planner may select actions that are safe in expectation but catastrophic for a fraction of the belief.

Existing approaches to risk-averse POMDP planning address this by modifying the value function---applying CVaR to the cumulative return \cite{chow2015risk}, treating CVaR as a constraint \cite{chow2014algorithms}, or using the Iterated CVaR (ICVaR) dynamic risk measure within the Bellman backup \cite{hardy2004iterated, ruszczynski2010risk, du2022provably, chu2014markov, pariente2026onlineriskaverseplanningpomdps, ahmadi2020risk}. These methods target trajectory-level risk, but they require the entire planning algorithm to be redesigned for the new objective, producing specialized variants \cite{pariente2026onlineriskaverseplanningpomdps} that cannot benefit from ongoing advances in expectation-based planners.

We observe that POMDPs introduce a specific source of risk that is distinct from trajectory-level risk: \emph{uncertainty about the current state}. At each planning step, the immediate cost depends on the unobserved state, and its expectation under the current belief can be low even when a significant fraction of the belief mass lies on high-cost states. We propose to directly target this risk by replacing the expected immediate cost with a risk-sensitive aggregate that evaluates each action against the worst fraction of states the belief places mass on, rather than averaging over the belief.

Because the standard expected cumulative return is retained as the objective, the resulting problem has standard MDP structure, so any existing POMDP planner can be made risk-sensitive by changing only the cost computation. The separation of risk from planning also has theoretical consequences: the CVaR cost is deterministic given the particle belief, so it cancels exactly in the error decomposition between theoretical and estimated Q-functions. Only standard Monte Carlo sampling error remains, enabling Hoeffding concentration bounds that are entirely independent of the risk level $\alpha$.

We build on the particle belief MDP (PB-MDP) framework of \citet{lim2023optimality}, which provides approximation guarantees between the true POMDP and the particle-based surrogate, to analyze the CVaR cost formulation. Our contributions are:
\begin{enumerate}
	\item \textbf{CVaR cost formulation for belief uncertainty.} We propose applying CVaR to the immediate cost under the belief, directly targeting state uncertainty in POMDPs. Unlike chance constraints, which require a binary danger classification, the CVaR cost operates on the cost function itself, capturing both frequency and severity of tail risk.
	\item \textbf{Compatibility with existing planners.} The formulation preserves standard MDP structure, enabling any expectation-based POMDP planner to become risk-sensitive by changing only the cost computation (\Cref{sec:cvar_cost_planning_algorithms}).
	\item \textbf{Risk-level-independent estimation bounds.} We establish finite-time performance guarantees for policy evaluation (\Cref{thm:cvar_cost_policy_eval_guarantees}) and sparse sampling (\Cref{thm:cvar_cost_sparse_sampling_guarantees}) that are \emph{independent of} the risk level $\alpha$---the estimation accuracy does not degrade as the planner becomes more risk-averse.
	\item \textbf{End-to-end POMDP guarantees.} We derive particle belief approximation bounds (\Cref{thm:concrete_approximation}) where the $1/\alpha$ factor enters as a multiplier rather than compounding, and end-to-end guarantees (\Cref{cor:end_to_end}) bounding the gap from the true POMDP value to the algorithmic estimates.
\end{enumerate}

\section{Preliminaries}

\subsection{Conditional Value-at-Risk}

Let $X$ be a random variable defined on a probability space $(\Omega, \mathcal{F}, P)$ with $\mathbb{E}[|X|] < \infty$. Denote the cumulative distribution function by $F_X(x) = P(X \leq x)$. The value at risk at confidence level $\alpha \in (0,1]$ is the $1 - \alpha$ quantile:
\begin{equation}
	\text{VaR}_\alpha(X) \triangleq \inf\{x \in \mathbb{R} : F_X(x) > 1 - \alpha\}.
\end{equation}
The conditional value at risk (CVaR) at confidence level $\alpha$ is defined as \cite{rockafellar2000optimization}
\begin{equation}\label{eq:cvar_def}
	\text{CVaR}_\alpha(X) \triangleq \inf_{w \in \mathbb{R}} \left\{ w + \frac{1}{\alpha} \mathbb{E}[(X - w)^+] \right\},
\end{equation}
where $(x)^+ = \max(x, 0)$. For a continuous $F_X$, CVaR admits the tail expectation representation $\text{CVaR}_\alpha(X) = \mathbb{E}[X \mid X > \text{VaR}_\alpha(X)]$.

Given i.i.d.\ samples $X_1, \ldots, X_n$ drawn from $F_X$, the CVaR estimator is \cite{brown2007large}
\begin{equation}\label{eq:cvar_estimator}
	\hat{C}_\alpha(X) \triangleq \inf_{w \in \mathbb{R}} \left\{ w + \frac{1}{n\alpha} \sum_{i=1}^{n} (X_i - w)^+ \right\}.
\end{equation}
For weighted samples $\{(X_i, \tilde{w}_i)\}_{i=1}^{n}$ with $\tilde{w}_i \geq 0$ and $\sum_i \tilde{w}_i = 1$, the weighted CVaR estimator is
\begin{equation}\label{eq:weighted_cvar_estimator}
	\hat{C}_\alpha(\{X_i\}, \{\tilde{w}_i\}) \triangleq \inf_{w \in \mathbb{R}} \left\{ w + \frac{1}{\alpha} \sum_{i=1}^{n} \tilde{w}_i\,(X_i - w)^+ \right\}.
\end{equation}

\subsection{Partially Observable Markov Decision Process}\label{sec:pomdp}

A finite-horizon Partially Observable Markov Decision Process (POMDP) is defined as the tuple $M = (X, A, Z, \mathcal{T}, O, c, \gamma, T, b_0)$, where $X$ is the state space, $A$ is the action space, $Z$ is the observation space, $\mathcal{T}(x_{t+1} \mid x_t, a_t) \triangleq P(x_{t+1} \mid x_t, a_t)$ is the state transition model, $O(z_t \mid x_t) \triangleq P(z_t \mid x_t)$ is the observation model, $c : \mathcal{B} \times A \to \mathbb{R}$ is the immediate cost defined on the belief space $\mathcal{B} \triangleq \Delta(X)$, $\gamma \in (0,1]$ is the discount factor, $T \in \mathbb{N}$ is the planning horizon, and $b_0$ is the initial belief.

Due to partial observability, the agent maintains a belief $b_t \in \mathcal{B}$, a probability distribution over the state space representing the posterior given the history of actions and observations $H_t = \{z_{1:t}, a_{0:t-1}, b_0\}$. The belief is updated recursively via Bayes' rule:
\begin{equation}\label{eq:belief_update}
	b_{t+1}(x_{t+1}) \propto O(z_{t+1} \mid x_{t+1}) \int_{x_t \in X} \mathcal{T}(x_{t+1} \mid x_t, a_t)\, b_t(x_t)\, dx_t.
\end{equation}
A policy $\pi_t : \mathcal{B} \to A$ maps beliefs to actions. Throughout this work, the belief-dependent cost is induced by a state-action cost $c : X \times A \to \mathbb{R}$ bounded as $c_{\min} \leq c(x, a) \leq c_{\max}$, with the conventional definition
\begin{equation}\label{eq:expected_belief_cost}
	c(b_t, a_t) \triangleq \mathbb{E}_{x \sim b_t}[c(x, a_t)].
\end{equation}
In \Cref{sec:problem_formulation} we replace the expectation in \eqref{eq:expected_belief_cost} with $\text{CVaR}_\alpha$.

\subsection{Particle Belief MDP}\label{sec:pb_mdp}

Since the belief space $\mathcal{B}$ is infinite-dimensional, we approximate it using a particle belief MDP (PB-MDP) \cite{lim2023optimality}. Formally, denote by $M_P \triangleq (\Sigma, A, \tau, \rho)$ the PB-MDP defined with respect to the POMDP $M$ and $N_p \in \mathbb{N}$, where
\begin{itemize}
	\item The state space over particle beliefs is
	\begin{equation*}
		\Sigma \triangleq \left\{ \bar{b} = \{(x_i, w_i)\}_{i=1}^{N_p} \;:\;
		x_i \in X,\; w_i \geq 0 \;\forall i,\;
		\textstyle\sum_{i=1}^{N_p} w_i > 0 \right\}.
	\end{equation*}
	\item $A$ is the action space as defined in the POMDP $M$.
	\item $\tau(\bar{b}_{t+1}|\bar{b}_{t}, a)$ is the belief transition probability, for $a\in A,\bar{b}\in \Sigma$. Each particle $x_i^{(\ell)}$ is sampled i.i.d.\ from a \emph{prior predictive distribution} $q_\ell$ (the state marginal at depth $\ell$ obtained by propagating through $\mathcal{T}$ without conditioning on observations) and assigned importance weight $w_i^{(\ell)} = \prod_{n=1}^{\ell} O(z_n \mid x_i^{(n)})$.
	\item $\rho(\bar{b},a) \triangleq \sum_{i=1}^{N_p} \tilde{w}_i c(x_i, a)$ is the expected cost, where $\tilde{w}_i \triangleq w_i / \sum_{j=1}^{N_p} w_j$ are the normalized weights.
\end{itemize}

\section{Problem Formulation}\label{sec:problem_formulation}

We replace the expected cost in the POMDP (\Cref{sec:pomdp}) and PB-MDP (\Cref{sec:pb_mdp}) with a CVaR cost that targets the worst-$\alpha$ fraction of states under the current belief. Let $\alpha \in (0,1]$. Define the CVaR cost under the true belief and the particle belief, respectively, as
\begin{equation}\label{eq:cvar_cost_def}
	c_\alpha(b_t, a) \triangleq \underset{x \sim b_t}{\text{CVaR}_\alpha}[c(x, a)], \qquad
	\rho_\alpha(\bar{b},a) \triangleq \hat{C}_\alpha(\{c(x_i,a)\}_{i=1}^{N_p}, \{\tilde{w}_i\}_{i=1}^{N_p}),
\end{equation}
where $\tilde{w}_i \triangleq w_i / \sum_{j=1}^{N_p} w_j$.

The CVaR cost targets a risk that is specific to partially observable settings: uncertainty about the current state. In a fully observable MDP, the immediate cost $c(x, a)$ is known exactly. In a POMDP, the cost depends on the unknown state $x \sim b_t$, and the expected cost $\mathbb{E}_{x \sim b_t}[c(x, a)]$ can mask dangerous states in the tail of the belief. By replacing the expectation with $\text{CVaR}_\alpha$, the planner pessimistically evaluates each action against the worst-$\alpha$ fraction of the belief, avoiding actions that rely on favorable state realizations.

The action-value functions under the POMDP and PB-MDP are
\begin{align}
	Q^{\pi}_{M, t}(b_t, a, \alpha) &\triangleq c_\alpha(b_t, a) + \gamma\,\mathbb{E}[V_{M, t+1}^\pi(b_{t+1}, \alpha) \mid b_t, a], \label{eq:cvar_cost_q_function_def} \\
	Q^{\pi}_{M_P, t}(\bar{b}_t, a, \alpha) &\triangleq \rho_\alpha(\bar{b}_t, a) + \gamma\,\mathbb{E}_{M_P}[V_{M_P, t+1}^\pi(\bar{b}_{t+1}, \alpha) \mid \bar{b}_t, a], \label{eq:q_hat_estimation}
\end{align}
with $V_{M,t}^\pi(b_t, \alpha) = Q^{\pi}_{M,t}(b_t, \pi(b_t), \alpha)$, $V_{M_P,t}^\pi(\bar{b}_t, \alpha) = Q^{\pi}_{M_P,t}(\bar{b}_t, \pi(\bar{b}_t), \alpha)$, and terminal conditions $V_{M,T+1}^\pi = V_{M_P,T+1}^\pi = 0$. \Cref{fig:cvar_cost_structure} illustrates the CVaR cost return structure across the planning horizon. The goals of this paper are twofold:

\begin{figure}[t]
	\centering
	\begin{tikzpicture}[>=stealth, every node/.style={font=\small}]
	% --- Timeline ---
	\draw[->, thick] (-0.3, 0) -- (13.5, 0);
	\foreach \t/\lab in {0/t, 3/t{+}1, 6/t{+}2, 10/T} {
		\draw[thick] (\t, -0.12) -- (\t, 0.12);
		\node[below=0.15cm, font=\small] at (\t, -0.12) {$\lab$};
	}
	\node[font=\small] at (8.0, 0.0) {$\cdots$};
	\node[below=0.15cm, font=\small] at (13.2, -0.12) {time};

	% --- Belief nodes above timeline ---
	\foreach \t/\blab in {0/\bar{b}_t, 3/\bar{b}_{t+1}, 6/\bar{b}_{t+2}, 10/\bar{b}_T} {
		\node[circle, draw, fill=blue!8, minimum size=0.75cm, font=\footnotesize] (b\t) at (\t, 1.2) {$\blab$};
		\draw[->, gray] (b\t) -- (\t, 0.15);
	}
	% Transition arrows between beliefs
	\draw[->, thick, blue!50!black] (b0) -- node[above, font=\footnotesize] {$\tau$} (b3);
	\draw[->, thick, blue!50!black] (b3) -- node[above, font=\footnotesize] {$\tau$} (b6);
	\draw[->, thick, blue!50!black, dashed] (b6) -- (8.0, 1.2);
	\draw[->, thick, blue!50!black, dashed] (8.0, 1.2) -- (b10);

	% --- CVaR cost boxes below timeline ---
	% Step t: 6 bars, 1 red tail
	\node[draw, rounded corners, fill=red!12, minimum width=2.2cm, minimum height=1.8cm, anchor=north] (cost0) at (0, -0.8) {};
	\node[font=\footnotesize, anchor=north] at (0, -0.85) {$\rho_\alpha(\bar{b}_t, a_t)$};
	\draw[fill=gray!30] (-0.75, -2.5) rectangle (-0.58, -2.1);
	\draw[fill=gray!30] (-0.53, -2.5) rectangle (-0.36, -1.8);
	\draw[fill=gray!30] (-0.31, -2.5) rectangle (-0.14, -1.55);
	\draw[fill=gray!30] (-0.09, -2.5) rectangle (0.08, -1.9);
	\draw[fill=gray!30] (0.13, -2.5) rectangle (0.30, -2.2);
	\draw[fill=red!50] (0.35, -2.5) rectangle (0.52, -1.7);
	\draw[dashed, red!70!black, thick] (0.30, -2.55) -- (0.30, -1.35);
	\node[font=\tiny, red!70!black, anchor=north west] at (0.32, -1.30) {worst $\alpha$};

	% Step t+1: ALL SAFE — tail is also in safe range
	\node[draw, rounded corners, fill=green!8, minimum width=2.2cm, minimum height=1.8cm, anchor=north] (cost1) at (3, -0.8) {};
	\node[font=\footnotesize, anchor=north] at (3, -0.85) {$\rho_\alpha(\bar{b}_{t+1}, a_{t+1})$};
	\draw[fill=gray!30] (2.35, -2.5) rectangle (2.52, -2.25);
	\draw[fill=gray!30] (2.57, -2.5) rectangle (2.74, -2.0);
	\draw[fill=gray!30] (2.79, -2.5) rectangle (2.96, -1.95);
	\draw[fill=gray!30] (3.01, -2.5) rectangle (3.18, -2.15);
	\draw[fill=gray!45] (3.23, -2.5) rectangle (3.40, -2.1);
	\draw[dashed, gray!50, thick] (3.18, -2.55) -- (3.18, -1.35);
	\node[font=\tiny, gray!50!black, anchor=north west] at (3.20, -1.30) {worst $\alpha$};

	% Step t+2
	\node[draw, rounded corners, fill=red!12, minimum width=2.2cm, minimum height=1.8cm, anchor=north] (cost2) at (6, -0.8) {};
	\node[font=\footnotesize, anchor=north] at (6, -0.85) {$\rho_\alpha(\bar{b}_{t+2}, a_{t+2})$};
	\draw[fill=gray!30] (5.25, -2.5) rectangle (5.38, -2.15);
	\draw[fill=gray!30] (5.43, -2.5) rectangle (5.56, -1.65);
	\draw[fill=gray!30] (5.61, -2.5) rectangle (5.74, -1.5);
	\draw[fill=gray!30] (5.79, -2.5) rectangle (5.92, -1.75);
	\draw[fill=gray!30] (5.97, -2.5) rectangle (6.10, -2.0);
	\draw[fill=gray!30] (6.15, -2.5) rectangle (6.28, -2.3);
	\draw[fill=red!50] (6.33, -2.5) rectangle (6.46, -1.6);
	\draw[dashed, red!70!black, thick] (6.28, -2.55) -- (6.28, -1.35);
	\node[font=\tiny, red!70!black, anchor=north west] at (6.30, -1.30) {worst $\alpha$};

	% Dots
	\node[font=\small] at (8.0, -1.8) {$\cdots$};

	% Step T: ALL SAFE — tail is also in safe range
	\node[draw, rounded corners, fill=green!8, minimum width=2.2cm, minimum height=1.8cm, anchor=north] (costT) at (10, -0.8) {};
	\node[font=\footnotesize, anchor=north] at (10, -0.85) {$\rho_\alpha(\bar{b}_T, a_T)$};
	\draw[fill=gray!30] (9.35, -2.5) rectangle (9.52, -2.0);
	\draw[fill=gray!30] (9.57, -2.5) rectangle (9.74, -1.85);
	\draw[fill=gray!30] (9.79, -2.5) rectangle (9.96, -1.7);
	\draw[fill=gray!30] (10.01, -2.5) rectangle (10.18, -1.95);
	\draw[fill=gray!45] (10.23, -2.5) rectangle (10.40, -2.05);
	\draw[dashed, gray!50, thick] (10.18, -2.55) -- (10.18, -1.35);
	\node[font=\tiny, gray!50!black, anchor=north west] at (10.20, -1.30) {worst $\alpha$};

	% --- Discount factors ---
	\node[font=\footnotesize, blue!60!black] at (0, -3.15) {$\gamma^0$};
	\node[font=\footnotesize, blue!60!black] at (3, -3.15) {$\gamma^1$};
	\node[font=\footnotesize, blue!60!black] at (6, -3.15) {$\gamma^2$};
	\node[font=\footnotesize, blue!60!black] at (10, -3.15) {$\gamma^{T-t}$};

	% --- Navigation scenes (below discount factors) ---
	% Robot macro: body + wheels + sensor dome
	% Usage: \drawrobot{x}{y}{scale}
	% Drawn at each scene center

	% Scene t: robot with spread belief, one dangerous particle
	\begin{scope}[shift={(0, -4.1)}, scale=0.7]
		\fill[gray!6] (-1.3, -1.0) rectangle (1.3, 1.0);
		\draw[gray!25, rounded corners=2pt] (-1.3, -1.0) rectangle (1.3, 1.0);
		% Robot: body + wheels + sensor
		\fill[blue!50!black, rounded corners=1pt] (-0.55, -0.15) rectangle (-0.15, 0.15); % body
		\fill[gray!60!black] (-0.55, -0.22) rectangle (-0.45, -0.15); % left wheel
		\fill[gray!60!black] (-0.25, -0.22) rectangle (-0.15, -0.15); % right wheel
		\fill[blue!35] (-0.35, 0.15) circle (0.07); % sensor dome
		\draw[blue!60!black, very thin] (-0.35, 0.22) -- (-0.35, 0.32); % antenna
		\fill[blue!60!black] (-0.37, 0.32) rectangle (-0.33, 0.34); % antenna tip
		% Safe particles
		\foreach \x/\y in {-0.7/0.5, -0.1/0.6, -0.6/-0.5, 0.15/0.35, -0.2/-0.4} {
			\fill[blue!40, opacity=0.7] (\x, \y) circle (0.065);
		}
		% Dangerous particle
		\fill[red!75!black] (0.8, 0.3) circle (0.085);
	\end{scope}

	% Scene t+1: ALL SAFE — robot moved right, all blue particles
	\begin{scope}[shift={(3, -4.1)}, scale=0.7]
		\fill[gray!6] (-1.3, -1.0) rectangle (1.3, 1.0);
		\draw[gray!25, rounded corners=2pt] (-1.3, -1.0) rectangle (1.3, 1.0);
		\fill[blue!50!black, rounded corners=1pt] (-0.15, -0.15) rectangle (0.25, 0.15);
		\fill[gray!60!black] (-0.15, -0.22) rectangle (-0.05, -0.15);
		\fill[gray!60!black] (0.15, -0.22) rectangle (0.25, -0.15);
		\fill[blue!35] (0.05, 0.15) circle (0.07);
		\draw[blue!60!black, very thin] (0.05, 0.22) -- (0.05, 0.32);
		\fill[blue!60!black] (0.03, 0.32) rectangle (0.07, 0.34);
		\foreach \x/\y in {-0.25/0.4, 0.2/0.5, -0.15/-0.35, 0.35/0.0, 0.55/-0.3} {
			\fill[blue!40, opacity=0.7] (\x, \y) circle (0.065);
		}
	\end{scope}

	% Scene t+2: robot further, belief spread wider
	\begin{scope}[shift={(6, -4.1)}, scale=0.7]
		\fill[gray!6] (-1.3, -1.0) rectangle (1.3, 1.0);
		\draw[gray!25, rounded corners=2pt] (-1.3, -1.0) rectangle (1.3, 1.0);
		\fill[blue!50!black, rounded corners=1pt] (-0.35, 0.15) rectangle (0.05, 0.45);
		\fill[gray!60!black] (-0.35, 0.08) rectangle (-0.25, 0.15);
		\fill[gray!60!black] (-0.05, 0.08) rectangle (0.05, 0.15);
		\fill[blue!35] (-0.15, 0.45) circle (0.07);
		\draw[blue!60!black, very thin] (-0.15, 0.52) -- (-0.15, 0.62);
		\fill[blue!60!black] (-0.17, 0.62) rectangle (-0.13, 0.64);
		\foreach \x/\y in {-0.55/0.65, 0.15/0.55, -0.6/0.2, 0.0/-0.1, 0.3/0.35} {
			\fill[blue!40, opacity=0.7] (\x, \y) circle (0.065);
		}
		\fill[red!75!black] (0.85, -0.55) circle (0.085);
	\end{scope}

	% Dots
	\node[font=\small] at (8.0, -4.1) {$\cdots$};

	% Scene T: ALL SAFE — robot near goal, all blue particles
	\begin{scope}[shift={(10, -4.1)}, scale=0.7]
		\fill[gray!6] (-1.3, -1.0) rectangle (1.3, 1.0);
		\draw[gray!25, rounded corners=2pt] (-1.3, -1.0) rectangle (1.3, 1.0);
		% Goal
		\node[star, star points=5, star point ratio=2.25, fill=green!50!black, minimum size=0.22cm, draw=green!70!black] at (0.85, 0.65) {};
		% Robot near goal
		\fill[blue!50!black, rounded corners=1pt] (0.35, 0.15) rectangle (0.75, 0.45);
		\fill[gray!60!black] (0.35, 0.08) rectangle (0.45, 0.15);
		\fill[gray!60!black] (0.65, 0.08) rectangle (0.75, 0.15);
		\fill[blue!35] (0.55, 0.45) circle (0.07);
		\draw[blue!60!black, very thin] (0.55, 0.52) -- (0.55, 0.62);
		\fill[blue!60!black] (0.53, 0.62) rectangle (0.57, 0.64);
		\foreach \x/\y in {0.35/0.6, 0.7/0.65, 0.45/0.0, 0.8/0.35, 0.2/0.4} {
			\fill[blue!40, opacity=0.7] (\x, \y) circle (0.065);
		}
	\end{scope}

	% Connecting arrows between scenes
	\draw[->, gray!40, thick] (0.95, -4.1) -- (1.45, -4.1);
	\draw[->, gray!40, thick] (3.95, -4.1) -- (4.45, -4.1);
	\draw[->, gray!40, thick, dashed] (6.95, -4.1) -- (7.5, -4.1);
	\draw[->, gray!40, thick, dashed] (8.5, -4.1) -- (9.05, -4.1);

	% --- Safe/Dangerous labels under histograms ---
	% Step t
	\draw[gray!40, thick] (-0.75, -2.65) -- (0.30, -2.65);
	\node[font=\tiny, gray!50!black] at (-0.22, -2.82) {safe};
	\draw[red!50!black, thick] (0.30, -2.65) -- (0.52, -2.65);
	\node[font=\tiny, red!60!black] at (0.41, -2.82) {danger};
	% Step t+1: all safe
	\draw[gray!40, thick] (2.35, -2.65) -- (3.40, -2.65);
	\node[font=\tiny, gray!50!black] at (2.88, -2.82) {all safe};
	% Step t+2
	\draw[gray!40, thick] (5.25, -2.65) -- (6.28, -2.65);
	\node[font=\tiny, gray!50!black] at (5.76, -2.82) {safe};
	\draw[red!50!black, thick] (6.28, -2.65) -- (6.46, -2.65);
	\node[font=\tiny, red!60!black] at (6.37, -2.82) {danger};
	% Step T: all safe
	\draw[gray!40, thick] (9.35, -2.65) -- (10.40, -2.65);
	\node[font=\tiny, gray!50!black] at (9.88, -2.82) {all safe};

	% --- Return formula ---
	\node[draw, rounded corners, fill=yellow!15, font=\footnotesize, text width=11.5cm, align=center, anchor=north] at (5.0, -5.3) {
		$G_t^\pi(\alpha) = \gamma^0 \, \rho_\alpha(\bar{b}_t, a_t) + \gamma^1 \, \rho_\alpha(\bar{b}_{t+1}, a_{t+1}) + \gamma^2 \, \rho_\alpha(\bar{b}_{t+2}, a_{t+2}) + \cdots + \gamma^{T-t} \, \rho_\alpha(\bar{b}_T, a_T)$
	};
	\end{tikzpicture}
	\caption{Structure of the CVaR cost return during navigation. \textbf{Top}: beliefs $\bar{b}_k$ evolve along the planning horizon. \textbf{Middle}: particle cost histograms---the worst-$\alpha$ fraction (right of the dashed threshold) determines the CVaR cost. At steps where all particles are safe (green boxes, e.g., $t{+}1$ and $T$), the tail is also safe; at steps with dangerous states (red boxes, e.g., $t$ and $t{+}2$), the tail captures high-cost particles. \textbf{Bottom}: corresponding navigation scenes---the robot maintains a particle belief (blue dots), with red dots marking particles in dangerous states. Safe steps show all-blue beliefs; dangerous steps show the outlier particles that drive the CVaR cost. The return $G_t^\pi(\alpha)$ sums these per-step CVaR costs $\rho_\alpha$ with discount $\gamma^{k-t}$.}
	\label{fig:cvar_cost_structure}
\end{figure}
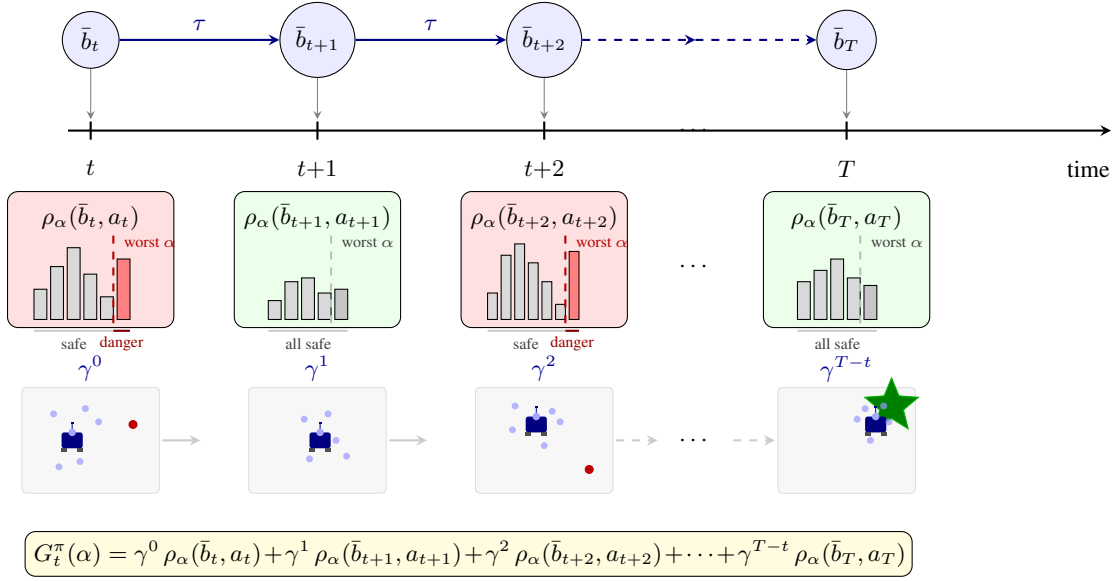
\begin{enumerate}
	\item \textbf{Planning.} Develop online planning algorithms---policy evaluation and sparse sampling---within the CVaR cost PB-MDP, and establish finite-time performance guarantees bounding $|Q^{\pi}_{M_P, t} - \hat{Q}^{\pi}_{M_P, t}|$ and $|V^{*}_{M_P, t} - \hat{V}^{*}_{M_P, t}|$.
	\item \textbf{Approximation guarantees.} Derive concentration inequalities bounding $|V_{M,t}^\pi - V_{M_P,t}^\pi|$, analogous to \citet{lim2023optimality} for the expectation-based setting, quantifying how particle approximation error propagates through the CVaR cost formulation.
\end{enumerate}

\subsection{Comparison with Prior CVaR Formulations}\label{sec:cvar_comparison}

Risk in POMDPs arises at two levels: (i) \emph{state uncertainty}---the agent does not know its current state, so the cost $c(x, a)$ is uncertain even for a fixed action; and (ii) \emph{trajectory uncertainty}---the sequence of future beliefs and costs is stochastic. Prior CVaR formulations target trajectory uncertainty; the CVaR cost formulation targets state uncertainty. We compare three approaches:

\begin{enumerate}
	\item \textbf{Static CVaR} targets the risk that the \emph{total return} is bad:
	\begin{equation}\label{eq:static_cvar}
		V_t^{\text{static}}(b_t, \alpha) = \underset{\text{over trajectories}}{\text{CVaR}_\alpha}\!\left[\sum_{k=t}^{T} \gamma^{k-t} c(b_k, a_k)\right].
	\end{equation}
	This captures trajectory-level tail risk, but does not specifically address the per-step danger from hidden states in the belief. The resulting Bellman equation requires tracking the entire return distribution, making dynamic programming intractable \cite{pflug2000some, chow2015risk}.

	\item \textbf{Dynamic CVaR (ICVaR)} targets the risk of \emph{transitioning to bad future beliefs}:
	\begin{equation}\label{eq:dynamic_cvar}
		V_t^{\text{dynamic}}(b_t, \alpha) = c(b_t, a_t) + \gamma\,\underset{\text{over } b_{t+1}}{\text{CVaR}_\alpha}\!\left[V_{t+1}^{\text{dynamic}}(b_{t+1}, \alpha)\right].
	\end{equation}
	The CVaR operator acts on the distribution of next beliefs $b_{t+1}$, capturing pessimism about \emph{which belief the agent will hold in the future}. It does not address the uncertainty about which state the agent is in \emph{now}---the immediate cost $c(b_t, a_t) = \mathbb{E}_{x \sim b_t}[c(x, a_t)]$ still averages over the current belief. The CVaR inside the recursion also introduces a $1/\alpha$ Lipschitz factor that compounds at each level, yielding bounds that scale as $O((1/\alpha)^{T-t})$ \cite{pariente2026onlineriskaverseplanningpomdps}.

	\item \textbf{CVaR cost (this work)} targets the risk from \emph{hidden dangerous states in the current belief}:
	\begin{equation}\label{eq:cvar_cost_vf}
		V_t^{\text{CVaR cost}}(b_t, \alpha) = c_\alpha(b_t, a_t) + \gamma\,\mathbb{E}\!\left[V_{t+1}^{\text{CVaR cost}}(b_{t+1}, \alpha)\right].
	\end{equation}
	The CVaR operator acts on the state distribution \emph{within} the belief, replacing $\mathbb{E}_{x \sim b_t}[c(x, a)]$ with $\text{CVaR}_\alpha[c(x, a)]$ for $x \sim b_t$. This makes the planner pessimistic about the worst-$\alpha$ fraction of states at each step, directly addressing state uncertainty. The value function retains standard MDP structure, and the $1/\alpha$ factor does not compound through the horizon.
\end{enumerate}

In summary, the three formulations differ in \emph{which uncertainty the CVaR operator targets}: the full return (static), the future belief transition (dynamic), or the current state within the belief (CVaR cost). Only the CVaR cost formulation directly addresses the per-step danger from partial observability, while also preserving compatibility with existing planners.

\section{CVaR Cost Planning Algorithms}\label{sec:cvar_cost_planning_algorithms}

Because the CVaR cost formulation yields a standard MDP with bounded costs $\rho_\alpha \in [c_{\min}, c_{\max}]$, existing planning algorithms and their guarantees apply directly. The only modification is replacing the expected cost with the CVaR cost. We instantiate policy evaluation on the CVaR cost PB-MDP via the Monte Carlo backup of \citet{kearns2002sparse}: for a given policy $\pi$, the Q-value and value estimators are:
\begin{align}
	\hat{Q}_{M_P, t}^\pi(\bar{b}_t,a,\alpha) &\triangleq \rho_\alpha(\bar{b}_t, a) + \frac{\gamma}{N_b}\sum_{i=1}^{N_b}\hat{V}_{M_P, t+1}^\pi(\bar{b}_{t+1}^i, \alpha), \label{eq:Qest}\\
	\hat{V}_{M_P, t}^\pi(\bar{b}_t,\alpha) &\triangleq \hat{Q}_{M_P, t}^{\pi}(\bar{b}_t,\pi(\bar{b}_t),\alpha), \label{eq:Vest}
\end{align}
where $\bar{b}_{t+1}^i$ are $N_b$ sampled successor beliefs and $\rho_\alpha(\bar{b}_t, a) = \hat{C}_\alpha(\{c(x_t^j, a)\}_{j=1}^{N_p}, \{\tilde{w}_t^j\}_{j=1}^{N_p})$ is the CVaR cost from the particle belief. For optimal planning, sparse sampling takes $\hat{V}^*_{M_P,t} = \min_{a \in A} \hat{Q}^*_{M_P,t}(\bar{b}_t, a, \alpha)$.

\subsection{Inherited Guarantees}

Define $\Delta\rho \triangleq c_{\max} - c_{\min}$, an upper bound on the range of the CVaR cost (since $\rho_\alpha \in [c_{\min}, c_{\max}]$). The CVaR cost $\rho_\alpha(\bar{b}, a)$ is deterministic given the particle belief, so it cancels exactly in the error decomposition:
\begin{equation}\label{eq:error_cancellation}
	Q_{M_P,t}^\pi - \hat{Q}_{M_P,t}^\pi = (\rho_\alpha + \gamma\,\mathbb{E}[V_{t+1}^\pi]) - \Big(\rho_\alpha + \frac{\gamma}{N_b}\sum_i \hat{V}_{t+1}^\pi\Big) = \gamma\Big(\mathbb{E}[V_{t+1}^\pi] - \frac{1}{N_b}\sum_i \hat{V}_{t+1}^\pi\Big).
\end{equation}
Only standard Monte Carlo error from sampling successor beliefs remains, so any finite-time guarantee for MDP planning applies unchanged. In particular, the sparse sampling analysis of \citet{kearns2002sparse} transfers directly. We state the inherited bounds with explicit constants (proofs in Appendices~\ref{proof:cvar_cost_policy_eval} and~\ref{proof:cvar_cost_sparse_sampling}). Let $S_t \triangleq \frac{(T-t)(T-t+1)}{2}$.

\begin{theorem}[Policy Evaluation]\label{thm:cvar_cost_policy_eval_guarantees}
	For any $\delta \in (0,1)$ and $N_b > 1$:
	\begin{equation}
		P \Bigg( \left| Q_{M_P, t}^\pi - \hat{Q}^\pi_{M_P, t} \right|
		\leq \Delta \rho \cdot S_t \sqrt{\frac{\ln\big(\frac{2(N_b^{T-t}-1)}{\delta(N_b-1)}\big)}{2N_b}} \Bigg) \geq 1-\delta.
	\end{equation}
\end{theorem}

\begin{theorem}[Sparse Sampling]\label{thm:cvar_cost_sparse_sampling_guarantees}
	For any $\delta \in (0,1)$ and $N_b > 1$:
	\begin{equation}
		P \Bigg( \left| V_{M_P, t}^* - \hat{V}^*_{M_P, t} \right|
		\leq \Delta \rho \cdot S_t \sqrt{\frac{\ln\big(\frac{2|A|((|A|N_b)^{T-t}-1)}{\delta(|A|N_b-1)}\big)}{2N_b}} \Bigg) \geq 1-\delta.
	\end{equation}
\end{theorem}

Neither bound depends on $\alpha$: the estimation accuracy does not degrade as the planner becomes more risk-averse. The estimation error scales as $O(S_t) = O(D^2)$; the $O(\gamma^D)$ scaling applies to the \emph{approximation} bounds (\Cref{thm:concrete_approximation}), which bound the gap between the true POMDP and the particle belief surrogate. The end-to-end guarantee (\Cref{cor:end_to_end}) combines both. In contrast, ICVaR-based methods \cite{pariente2026onlineriskaverseplanningpomdps} introduce a $1/\alpha$ Lipschitz factor at each level, yielding estimation error that scales as $O((1/\alpha)^{T-t})$ and becomes vacuous for small $\alpha$ or long horizons.

\section{Particle Belief Approximation Guarantees}\label{sec:approximation_guarantees}

We now bound the gap between the true POMDP and the particle belief surrogate. We couple the true and particle beliefs by driving both from the same state trajectory $\omega$, so that $b_\ell(\omega)$ and $\bar{b}_\ell(\omega)$ are deterministic functions of $\omega$ (formal definition in \Cref{sec:ss_omega_framework}). In this framework, particle trajectories are generated independently---each particle propagates through $\mathcal{T}$ and accumulates observation likelihood weights $w_i^{(\ell)} = \prod_{n=1}^{\ell} O(z_n \mid x_i^{(n)})$---so at each depth $\ell$, the particles are i.i.d.\ from the prior predictive distribution. Conditions (i)--(vi) of \citet{lim2023optimality} require, in summary, that the observation model has bounded likelihood ratios, the prior predictive proposal covers the posterior (i.e., the R\'{e}nyi divergence $d_\infty(b_\ell \| q_\ell) \leq d_\infty^{\max} < \infty$ at every depth $\ell$), and that particles are propagated independently through the transition model. These conditions are standard for importance-weighted particle filters and, despite the absence of resampling, the bounded R\'{e}nyi divergence caps weight inflation and yields the uniform weighted-CDF concentration in \Cref{lem:weighted_dkw}.

The CVaR cost formulation requires the particle belief to approximate not just expectations (as in the standard PB-MDP) but the entire cost CDF, since CVaR depends on the distributional shape. We measure this via the \emph{distributional discrepancy} at depth $\ell$:
\begin{equation}\label{eq:dist_discrepancy_main}
	\varepsilon(\ell, N_p, \omega) \triangleq \sup_{a \in A}\, \sup_{z \in \mathbb{R}}\, \left| P_{b_\ell(\omega)}(c(x, a) \leq z) - \sum_{i=1}^{N_p} \tilde{w}_i^{(\ell)}\, \mathbf{1}(c(x_i^{(\ell)}, a) \leq z) \right|,
\end{equation}
the CDF sup-norm between the cost distribution under the true belief and the weighted empirical distribution under the particle belief.

The key insight is that the CVaR cost approximation error at each depth $\ell$ is controlled by two factors: the \emph{number of particles} $N_p$, which determines how well the weighted empirical CDF approximates the true cost CDF, and the \emph{depth} $\ell$, since each propagation step through the transition model accumulates error in the particle belief. CVaR amplifies CDF discrepancy by a factor of $\Delta\rho/\alpha$ via its $1/\alpha$-Lipschitz property with respect to the Wasserstein-1 distance (\Cref{lem:cvar_cost_sensitivity}), but this factor enters once as a multiplier rather than compounding across depths. The following proposition makes this precise:

\begin{proposition}[Abstract Approximation Bound]\label{prop:cvar_cost_approximation}
	Let $\pi$ be any policy, $b_t$ a realizable belief, $\bar{b}_t$ a corresponding particle belief with $N_p$ particles, and $a$ any action. Let $\omega$ be a state trajectory coupling the true and particle beliefs as in \Cref{sec:ss_omega_framework}. Suppose that $\varepsilon(\ell, N_p, \omega) < \alpha$ for all $\ell \in \{t, \ldots, T\}$ and all $\omega$ in an event $\mathcal{G}$ with $P(\mathcal{G}) \geq 1 - \delta_0$. Then, with probability at least $1 - \delta_0$:
	\begin{equation}\label{eq:approximation_bound}
		\left| Q_{M,t}^\pi(b_t, a, \alpha) - Q_{M_P,t}^\pi(\bar{b}_t, a, \alpha) \right|
		\leq \frac{\Delta\rho}{\alpha} \sum_{\ell=t}^{T} \gamma^{\ell - t}\, \varepsilon(\ell, N_p),
	\end{equation}
	where $\varepsilon(\ell, N_p) \triangleq \sup_{\omega \in \mathcal{G}} \varepsilon(\ell, N_p, \omega)$.
\end{proposition}

\begin{proof}
	Under the coupling in \Cref{sec:ss_omega_framework}, both trajectories are driven by the same $\omega$. Fix $\omega \in \mathcal{G}$ and let $a_t = a$, $a_\ell = \pi(h_\ell)$ for $\ell > t$. Define
	$R_M(\omega) \triangleq \sum_{\ell=t}^{T} \gamma^{\ell-t} c_\alpha(b_\ell(\omega), a_\ell)$ and
	$R_{M_P}(\omega) \triangleq \sum_{\ell=t}^{T} \gamma^{\ell-t} \rho_\alpha(\bar{b}_\ell(\omega), a_\ell)$.
	By the triangle inequality and \Cref{lem:cvar_cost_sensitivity} (since $\varepsilon(\ell, N_p) < \alpha$):
	\begin{equation}
		|R_M(\omega) - R_{M_P}(\omega)|
		\leq \sum_{\ell=t}^{T} \gamma^{\ell-t} |c_\alpha(b_\ell(\omega), a_\ell) - \rho_\alpha(\bar{b}_\ell(\omega), a_\ell)|
		\leq \frac{\Delta\rho}{\alpha} \sum_{\ell=t}^{T} \gamma^{\ell-t} \varepsilon(\ell, N_p).
	\end{equation}
	The right-hand side is $\omega$-independent, so $|Q_{M,t}^\pi - Q_{M_P,t}^\pi| = |\mathbb{E}_\omega[R_M - R_{M_P}]| \leq \mathbb{E}_\omega[|R_M - R_{M_P}|]$ yields \eqref{eq:approximation_bound} with probability at least $P(\mathcal{G}) \geq 1 - \delta_0$.
\end{proof}

Instantiating the discrepancy $\varepsilon(\ell, N_p)$ via a weighted DKW bound (\Cref{lem:weighted_dkw} in the appendix) and a union bound over $D+1$ time steps yields a fully concrete bound:

\begin{theorem}[POMDP--PB-MDP Approximation]\label{thm:concrete_approximation}
	Let $\pi$ be any policy, $b_t$ a realizable belief, $\bar{b}_t$ a corresponding particle belief with $N_p$ particles sampled from the prior predictive $q_\ell$ (\Cref{sec:pb_mdp}), and $a$ any action. Under conditions (i)--(vi) of \citet{lim2023optimality} with R\'{e}nyi divergence bound $d_\infty(b_\ell \| q_\ell) \leq d_\infty^{\max} < \infty$ at every depth $\ell$, with probability at least $1 - \delta_0$:
	\begin{equation}\label{eq:approx_concrete}
		\left| Q_{M,t}^\pi(b_t, a, \alpha) - Q_{M_P,t}^\pi(\bar{b}_t, a, \alpha) \right| \leq \frac{\Delta\rho\, d_\infty^{\max}}{\alpha} \cdot G_{D+1} \cdot \varepsilon_{\mathrm{approx}}(N_p, D, \delta_0),
	\end{equation}
	where $D = T - t$, $G_n = (1 - \gamma^n)/(1-\gamma)$, and $\varepsilon_{\mathrm{approx}}$ is the smallest $\varepsilon > 0$ satisfying
	\begin{equation}\label{eq:eps_approx_implicit}
		2\!\left(\!\left\lceil \frac{4d_\infty^{\max}}{\varepsilon} \right\rceil + 2\right)(D+1)\exp\!\left(-\frac{N_p\,\varepsilon^2}{8(d_\infty^{\max})^2}\right) \leq \delta_0.
	\end{equation}
	In particular, $\varepsilon_{\mathrm{approx}} = O\!\left(d_\infty^{\max}\sqrt{\ln(d_\infty^{\max} D/(\varepsilon\,\delta_0))/N_p}\right)$, and the condition $N_p = \tilde{\Omega}((d_\infty^{\max}/\alpha)^2)$ ensures $\varepsilon_{\mathrm{approx}} < \alpha$.
\end{theorem}

\begin{proof}[Proof sketch]
The Q-gap decomposes (\Cref{prop:cvar_cost_approximation}) into per-step CVaR-cost discrepancies $|c_\alpha(b_\ell, a_\ell) - \rho_\alpha(\bar{b}_\ell, a_\ell)|$ summed with discount $\gamma^{\ell-t}$. Each is controlled by the cost-CDF discrepancy $\varepsilon(\ell, N_p)$ from \eqref{eq:dist_discrepancy_main}: the $1/\alpha$-Lipschitz property of $\text{CVaR}_\alpha$ in $W_1$ (\Cref{lem:cvar_cost_sensitivity}) gives $|c_\alpha(b_\ell, a_\ell) - \rho_\alpha(\bar{b}_\ell, a_\ell)| \leq (\Delta\rho/\alpha)\,\varepsilon(\ell, N_p)$ , with $\alpha$ appearing nowhere else in the bound, and the weighted DKW (\Cref{lem:weighted_dkw}) gives $\varepsilon(\ell, N_p) \leq \tilde O(d_\infty^{\max}/\sqrt{N_p})$. A per-depth budget $\delta_0/(D+1)$ and discount sum $G_{D+1}$ assemble these into \eqref{eq:approx_concrete}; full details in \Cref{proof:concrete_approximation}.
\end{proof}

The bound reflects the structure identified in \Cref{prop:cvar_cost_approximation}: the $1/\alpha$ factor from CVaR sensitivity and the $d_\infty^{\max}$ factor from importance weighting enter as multipliers, the geometric sum $G_{D+1}$ captures the discounted accumulation across depths, and $\varepsilon_{\mathrm{approx}} = \tilde{O}(1/\sqrt{N_p})$ is the uniform convergence rate from the weighted DKW bound (\Cref{lem:weighted_dkw}). An alternative bound that explicitly separates CVaR cost discrepancy from observation distribution mismatch is given in \Cref{sec:alternative_approx}.

Because the CVaR cost formulation preserves MDP structure, any PB-MDP planner with estimation guarantees immediately yields end-to-end POMDP guarantees:

\begin{corollary}[End-to-End POMDP Guarantee]\label{cor:end_to_end}
	Let $\hat{Q}_{M_P,t}$ be any estimate of $Q_{M_P,t}$ satisfying $|Q_{M_P,t} - \hat{Q}_{M_P,t}| \leq \varepsilon_{\mathrm{est}}$ with probability at least $1 - \delta$. Under the conditions of \Cref{thm:concrete_approximation}, with probability at least $1 - \delta - \delta_0$:
	\begin{equation}\label{eq:end_to_end}
		|Q_{M,t}^\pi - \hat{Q}_{M_P,t}^\pi| \leq \underbrace{\frac{\Delta\rho\, d_\infty^{\max}}{\alpha} \cdot G_{D+1} \cdot \varepsilon_{\mathrm{approx}}}_{\text{approximation (\Cref{thm:concrete_approximation})}} + \;\varepsilon_{\mathrm{est}}.
	\end{equation}
\end{corollary}

\begin{proof}
	Triangle inequality: $|Q_{M,t}^\pi - \hat{Q}_{M_P,t}^\pi| \leq |Q_{M,t}^\pi - Q_{M_P,t}^\pi| + |Q_{M_P,t}^\pi - \hat{Q}_{M_P,t}^\pi|$, bounded by \Cref{thm:concrete_approximation} and the assumed estimation guarantee. A union bound gives the combined probability.
\end{proof}

\begin{remark}[Instantiations]\label{rem:end_to_end_instantiations}
	\Cref{cor:end_to_end} applies to any PB-MDP planner with finite-time guarantees. For the algorithms in this paper:
	\begin{itemize}
		\item \textbf{Policy evaluation} (\Cref{thm:cvar_cost_policy_eval_guarantees}): $\varepsilon_{\mathrm{est}} = \Delta\rho \cdot S_t \sqrt{\ln\!\left(\frac{2(N_b^{D}-1)}{\delta(N_b-1)}\right)\!/(2N_b)}$.
		\item \textbf{Sparse sampling} (\Cref{thm:cvar_cost_sparse_sampling_guarantees}): $\varepsilon_{\mathrm{est}} = \Delta\rho \cdot S_t \sqrt{\ln\!\left(\frac{2|A|((|A|N_b)^{D}-1)}{\delta(|A|N_b-1)}\right)\!/(2N_b)}$.
	\end{itemize}
	Both estimation bounds are independent of $\alpha$, so the only $\alpha$-dependence in the end-to-end guarantee comes from the approximation term. For sparse sampling, $\varepsilon_{\mathrm{approx}}$ acquires an additional $|A|$ factor inside the logarithm, since the CDF discrepancy must hold for all actions at each step.
\end{remark}

\section{Experiments}\label{sec:experiments}

We evaluate CVaR cost planning on four POMDP environments with dangerous regions, implemented using the POMDPPlanners package \citep{pariente2026pomdpplanners}. To demonstrate the modularity of the approach, we apply the CVaR cost modification to two structurally different online planners: PFT-DPW \citep{sunberg2018online} and POMCPOW \citep{sunberg2018online}. We compare six planners:
\begin{itemize}
	\item \textbf{PFT-DPW} / \textbf{POMCPOW}: standard planners using expected immediate cost;
	\item \textbf{CVaRCost-PFT-DPW} / \textbf{CVaRCost-POMCPOW}: our variants using $\text{CVaR}_\alpha$ of per-particle immediate costs;
	\item \textbf{ICVaR-PFT-DPW} / \textbf{ICVaR-POMCPOW}: Iterated CVaR (dynamic risk) variants \cite{pariente2026onlineriskaverseplanningpomdps} that apply CVaR to future belief transitions while retaining the expected immediate cost.
\end{itemize}
All planners are evaluated with the \emph{CVaR cost return} $G = \sum_{k=0}^{T-1} \gamma^k \text{CVaR}_\alpha(\{c(x^i_k, a_k)\}, \{w^i_k\})$ computed from the planner's own particle belief at each step (i.e., the same belief used for action selection), with $\alpha = 0.1$ and $\gamma = 0.95$.

\paragraph{Environments.}
\textbf{LightDark}: a 5$\times$5 continuous navigation task with stochastic obstacle penalties ($-400$ with probability $0.1$; 50 steps).
\textbf{PushPOMDP}: a 7$\times$7 continuous push task with an obstacle at $(5,3)$ and stochastic obstacle penalty ($-50$ with probability $0.1$; 30 steps).
\textbf{ContinuousLaserTag}: a continuous pursuit-evasion task with two dangerous areas where the reward is stochastic: $+10$ with probability $0.9$ and $-90$ with probability $0.1$ (expected value $0$), so the expected-cost planner has no incentive to avoid the danger zone while CVaR at $\alpha = 0.1$ captures the $-90$ tail (50 steps).
\textbf{PacMan}: a 7$\times$7 grid pursuit-evasion task with 2 ghosts, where PacMan slips with probability $0.1$ and ghost collision incurs a penalty of $-100$ with probability $0.2$ (expected penalty $-20$), creating a bimodal cost distribution that CVaR captures (50 steps).
All planners use $N_p = 200$ particles and a 4-second planning timeout per step. Hyperparameters (depth, exploration constant, progressive widening coefficients) were tuned independently per (environment, planner) pair using Optuna (50 trials, 10 episodes per trial), minimizing the CVaR cost return for all planners to ensure a fair comparison (\Cref{sec:hyperparameter_tuning}). Note that the sufficient condition in \Cref{thm:concrete_approximation} requires $N_p = \tilde{\Omega}((d_\infty^{\max}/\alpha)^2)$, which at $\alpha = 0.1$ is more demanding than the $N_p = 200$ used here; as is common for concentration-based sufficient conditions, the theoretical requirement is conservative and the experiments show strong performance well below it.

\paragraph{Results.}
\Cref{tab:cvar_cost_results} summarizes the results over 500 episodes per configuration.

\begin{table}[h]
	\centering
	\caption{CVaR cost return (mean $\pm$ 95\% CI), dangerous encounters, and goal rate across 500 episodes. Lower CVaR cost return and fewer dangerous encounters indicate safer behavior. Dangerous encounters are: obstacle hits (LightDark), obstacle collisions (PushPOMDP), danger area steps (ContinuousLaserTag), and ghost collisions (PacMan). Goal rate is not applicable (---) for PacMan, which is an infinite-horizon pursuit-evasion task with no terminal goal state.}\label{tab:cvar_cost_results}
	\small
	\begin{tabular}{llccc}
		\toprule
		Environment & Planner & CVaR Cost Return & Danger Encounters & Goal Rate \\
		\midrule
		\multirow{6}{*}{LightDark}
		& PFT-DPW            & $777.7 \pm 26.8$ & $2.16 \pm 0.09$ & $100\%$ \\
		& ICVaR-PFT-DPW      & $250.5 \pm 24.1$ & $0.97 \pm 0.11$ & $100\%$ \\
		& CVaRCost-PFT-DPW   & $\mathbf{144.8 \pm 11.1}$ & $\mathbf{0.61 \pm 0.07}$ & $100\%$ \\
		& POMCPOW            & $820.3 \pm 26.7$ & $2.28 \pm 0.09$ & $100\%$ \\
		& ICVaR-POMCPOW      & $58.6 \pm 6.6$ & $0.31 \pm 0.06$ & $100\%$ \\
		& CVaRCost-POMCPOW   & $\mathbf{87.9 \pm 6.5}$ & $\mathbf{0.45 \pm 0.06}$ & $100\%$ \\
		\midrule
		\multirow{6}{*}{PushPOMDP}
		& PFT-DPW            & $47.4 \pm 3.0$  & $2.83 \pm 0.17$ & $68 \pm 4\%$ \\
		& ICVaR-PFT-DPW      & $76.8 \pm 4.4$  & $3.60 \pm 0.22$ & $59 \pm 4\%$ \\
		& CVaRCost-PFT-DPW   & $\mathbf{21.0 \pm 0.8}$   & $\mathbf{1.78 \pm 0.11}$ & $27 \pm 4\%$ \\
		& POMCPOW            & $31.2 \pm 2.3$   & $2.79 \pm 0.19$ & $64 \pm 4\%$ \\
		& ICVaR-POMCPOW      & $22.7 \pm 1.3$   & $2.18 \pm 0.14$ & $71 \pm 4\%$ \\
		& CVaRCost-POMCPOW   & $\mathbf{17.7 \pm 1.0}$   & $\mathbf{1.66 \pm 0.11}$ & $33 \pm 4\%$ \\
		\midrule
		\multirow{6}{*}{Cont.\ LaserTag}
		& PFT-DPW            & $119.9 \pm 11.6$ & $2.38 \pm 0.26$ & $100\%$ \\
		& ICVaR-PFT-DPW      & $77.6 \pm 9.1$ & $1.59 \pm 0.20$ & $100\%$ \\
		& CVaRCost-PFT-DPW   & $\mathbf{12.9 \pm 1.4}$   & $\mathbf{0.14 \pm 0.04}$ & $100\%$ \\
		& POMCPOW            & $188.0 \pm 8.0$  & $4.03 \pm 0.14$ & $100\%$ \\
		& ICVaR-POMCPOW      & $8.1 \pm 0.6$ & $0.02 \pm 0.02$ & $100\%$ \\
		& CVaRCost-POMCPOW   & $\mathbf{5.4 \pm 0.5}$    & $\mathbf{0.04 \pm 0.03}$ & $100\%$ \\
		\midrule
		\multirow{6}{*}{PacMan}
		& PFT-DPW            & $57.3 \pm 3.1$   & $0.73 \pm 0.04$ & --- \\
		& ICVaR-PFT-DPW      & $57.5 \pm 3.1$   & $0.76 \pm 0.04$ & --- \\
		& CVaRCost-PFT-DPW   & $\mathbf{39.5 \pm 1.8}$   & $\mathbf{0.53 \pm 0.04}$ & --- \\
		& POMCPOW            & $95.6 \pm 5.0$  & $0.90 \pm 0.03$ & --- \\
		& ICVaR-POMCPOW      & $43.4 \pm 2.0$   & $0.63 \pm 0.04$ & --- \\
		& CVaRCost-POMCPOW   & $\mathbf{37.7 \pm 1.5}$   & $\mathbf{0.57 \pm 0.04}$ & --- \\
		\bottomrule
	\end{tabular}
\end{table}

Across all four environments, CVaR cost planners outperform their expected-cost counterparts on both metrics, and outperform ICVaR baselines in most (environment, base-planner) cells; the exceptions are ICVaR-POMCPOW on LightDark and ContinuousLaserTag, discussed below. The effect is most striking in ContinuousLaserTag, where the stochastic danger structure (expected reward $0$ in danger zones) is invisible to expected-cost planners: CVaRCost-POMCPOW reduces CVaR cost return by 97\% ($5.4$ vs.\ $188.0$) and nearly eliminates danger area entries ($0.04$ vs.\ $4.03$). The CVaR cost modification is effective for both PFT-DPW and POMCPOW, demonstrating the modularity of the approach.

\paragraph{Risk--reward trade-off in PushPOMDP.} PushPOMDP exhibits a sharp safety--task-completion trade-off: CVaRCost variants reduce obstacle collisions by roughly 40\% ($1.78$ and $1.66$ vs.\ $2.79$--$2.83$ for the baselines) but reach the goal in only $27$--$33\%$ of episodes versus $59$--$71\%$ for the baselines. The goal in PushPOMDP requires routing the pushed object near the obstacle at $(5,3)$, and at $\alpha = 0.1$ the CVaR cost penalizes any belief with non-negligible mass on collision states, leading the planner to detour or abort the push. This is the expected behavior of a per-step tail-risk objective on a task where progress and risk are spatially coupled, and it reflects the user's choice of $\alpha$ rather than a failure of the planner; intermediate $\alpha$ values trade collisions for goal rate along this frontier.

\paragraph{CVaR cost vs.\ ICVaR.} ICVaR optimizes a different objective---dynamic risk over future belief transitions. With tuned hyperparameters, ICVaR planners can achieve competitive CVaR cost return in some settings (e.g., ICVaR-POMCPOW scores $58.6$ in LightDark and $8.1$ in ContinuousLaserTag), but this comes at the cost of requiring environment-specific tuning. The CVaR cost formulation, by contrast, directly targets the per-step danger from hidden states in the current belief, and achieves strong performance more consistently across environments and planner types.

\paragraph{Extended results and ablations.} \Cref{sec:extended_results} presents an extended table including the expected return $\mathbb{E}[G]$ and static $\text{CVaR}_{0.1}$ of the cumulative return, showing that CVaR cost planners also achieve the best static CVaR across all environments. \Cref{sec:ablation_studies} studies sensitivity to $\alpha$ and $N_p$: at $\alpha = 1$ the CVaR cost planner recovers risk-neutral performance, and the advantage is robust across particle counts from 50 to 500.

\section{Conclusion}

We introduced a modular approach to risk-sensitive POMDP planning: by applying CVaR to the immediate cost while retaining the standard expected cumulative return, any state-of-the-art POMDP planner becomes risk-sensitive through a change in the cost definition alone, with no algorithmic modification. This separation of risk from planning is backed by estimation bounds that are independent of the risk level $\alpha$, particle belief approximation bounds where the $1/\alpha$ factor enters additively, and a formal connection to chance-constrained planning.

\paragraph{Limitations and future work.}
The CVaR cost formulation captures per-step tail risk but does not directly control the tail of the cumulative return; when costs exhibit strong temporal correlations, per-step CVaR may underestimate trajectory-level risk. The modularity principle extends naturally beyond CVaR to any coherent risk measure applied to the immediate cost. Extending the formal analysis to MCTS-based planners with adaptive sampling and investigating adaptive $\alpha$ selection are promising directions.

\paragraph{Broader Impact.}
This work aims to make autonomous systems safer by accounting for tail risks in safety-critical domains; we foresee no negative societal impacts.

\section*{Acknowledgments}
This work was supported by the Israel Ministry of Innovation, Science and Technology.

\bibliography{references}
\bibliographystyle{plainnat}

%%%%%%%%%%%%%%%%%%%%%%%%%%%%%%%%%%%%%%%%%%%%%%%%%%%%%%%%%%%%%%%%%%%%%%%%%%%%%%%
%%%%%%%%%%%%%%%%%%%%%%%%%%%%%%%%%%%%%%%%%%%%%%%%%%%%%%%%%%%%%%%%%%%%%%%%%%%%%%%
% APPENDIX
%%%%%%%%%%%%%%%%%%%%%%%%%%%%%%%%%%%%%%%%%%%%%%%%%%%%%%%%%%%%%%%%%%%%%%%%%%%%%%%
%%%%%%%%%%%%%%%%%%%%%%%%%%%%%%%%%%%%%%%%%%%%%%%%%%%%%%%%%%%%%%%%%%%%%%%%%%%%%%%
\newpage
\appendix

\section{Proof of Theorem~\ref{thm:cvar_cost_policy_eval_guarantees}}\label{proof:cvar_cost_policy_eval}

\begin{theorem}[Restated]\label{thm:cvar_cost_policy_eval_full}
	Let $\delta \in (0,1)$, and let $\bar{b}_t = \{(x_t^i, w_t^i)\}_{i=1}^{N_p}$ be the particle belief at time $t$. With $\Delta\rho = c_{\max} - c_{\min}$, let $S_t \triangleq \sum_{k=1}^{T-t} k = \frac{(T-t)(T-t+1)}{2}$. If $N_b > 1$, then
	\begin{equation}
		P \! \left( \left| Q_{M_P, t}^\pi(\bar{b}_t, a, \alpha) - \hat{Q}^\pi_{M_P, t}(\bar{b}_t, a, \alpha) \right| \leq \Delta\rho \cdot S_t \sqrt{\frac{\ln\!\big(\frac{2(N_b^{T-t}-1)}{\delta(N_b-1)}\big)}{2N_b}} \right) \geq 1-\delta.
	\end{equation}
\end{theorem}

\begin{proof}
	The proof proceeds by induction on the horizon $T-t$. We first decompose the estimation error, then establish recursive probabilistic bounds, and finally derive explicit error bounds by unrolling the recursive relations.

	\textbf{Error Decomposition.} The estimation error decomposes as:
	\begin{equation}\label{eq:error_decomp}
		\begin{aligned}
			&Q_{M_P, t}^\pi(\bar{b}_t, a, \alpha) - \hat{Q}^\pi_{M_P, t}(\bar{b}_t, a, \alpha) \\
			&= \rho_\alpha(\bar{b}_t, a) + \gamma\,\mathbb{E}_{M_P}[V_{M_P, t+1}^\pi(\bar{b}_{t+1}, \alpha)|\bar{b}_t, a] \\
			&\quad - \rho_\alpha(\bar{b}_t, a) - \frac{\gamma}{N_b}\sum_{i=1}^{N_b} \hat{V}_{M_P, t+1}^\pi(\bar{b}_{t+1}^i, \alpha) \\
			&= \gamma\bigg(\underbrace{\mathbb{E}_{M_P}[V_{M_P, t+1}^\pi] - \frac{1}{N_b}\sum_{i=1}^{N_b} V_{M_P, t+1}^\pi(\bar{b}_{t+1}^i, \alpha)}_{\text{(I) Monte Carlo sampling error}} \\
			&\quad + \underbrace{\frac{1}{N_b}\sum_{i=1}^{N_b} \left( V_{M_P, t+1}^\pi(\bar{b}_{t+1}^i, \alpha) - \hat{V}_{M_P, t+1}^\pi(\bar{b}_{t+1}^i, \alpha) \right)}_{\text{(II) Propagated estimation error}}\bigg).
		\end{aligned}
	\end{equation}
	Note that the CVaR cost $\rho_\alpha(\bar{b}_t, a)$ cancels exactly since it is computed deterministically from the particle belief in both the theoretical and estimated Q-functions. Since $\gamma \leq 1$, we use $\gamma \leq 1$ to obtain the (slightly looser) bound $|Q_{M_P,t}^\pi - \hat{Q}_{M_P,t}^\pi| \leq |\text{(I)}| + |\text{(II)}|$.

	\textbf{Inductive Hypothesis.} We prove by induction that for all $t \in \{0, \ldots, T\}$:
	\begin{equation}\label{eq:induction_hyp}
		P\left( |Q_{M_P, t}^\pi(\bar{b}_t, a, \alpha) - \hat{Q}^\pi_{M_P, t}(\bar{b}_t, a, \alpha)| \leq \theta_t \right) \geq 1 - \eta_t,
	\end{equation}
	where the sequences $\{\theta_t\}$ and $\{\eta_t\}$ satisfy the recursions:
	\begin{align}
		\theta_t &= \epsilon_t + \theta_{t+1}, \quad \theta_T = 0, \label{eq:theta_recursion}\\
		\eta_t &= \delta_0 + N_b \eta_{t+1}, \quad \eta_T = 0, \label{eq:eta_recursion}
	\end{align}
	with $\epsilon_t = (T-t)\Delta\rho \sqrt{\frac{\ln(2/\delta_0)}{2N_b}}$ and $\delta_0$ to be determined.

	\textbf{Base Case ($t = T$).} At the terminal time, $V_{M_P, T+1}^\pi = \hat{V}_{M_P, T+1}^\pi = 0$ by definition, so:
	\begin{equation}
		Q_{M_P, T}^\pi(\bar{b}_T, a, \alpha) - \hat{Q}^\pi_{M_P, T}(\bar{b}_T, a, \alpha) = \rho_\alpha(\bar{b}_T, a) - \rho_\alpha(\bar{b}_T, a) = 0.
	\end{equation}
	Thus, $|Q - \hat{Q}| = 0 \leq \theta_T = 0$ with probability $1 \geq 1 - \eta_T = 1$.

	\textbf{Inductive Step.} Assume the bound \eqref{eq:induction_hyp} holds for time $t+1$. We prove it holds for time $t$.

	\textit{Bounding Term (I):} Conditioned on $\bar{b}_t$, the samples $\{V_{M_P, t+1}^\pi(\bar{b}_{t+1}^i, \alpha)\}_{i=1}^{N_b}$ are i.i.d.\ draws from the distribution of $V_{M_P, t+1}^\pi(\bar{b}_{t+1}, \alpha)$. The value function $V_{M_P, t+1}^\pi$ has range at most $(T-t)\Delta\rho$ since it is the sum of at most $T-t$ CVaR costs, each with range bounded by $\Delta\rho$. By Hoeffding's inequality:
	\begin{equation}\label{eq:hoeffding_bound}
		P\!\left( \left| \mathbb{E}[V_{M_P, t+1}^\pi] - \tfrac{1}{N_b}\textstyle\sum_{i=1}^{N_b} V_{M_P, t+1}^\pi(\bar{b}_{t+1}^i) \right| \leq \epsilon_t \right) \geq 1 - 2\exp\!\left( -\tfrac{2N_b \epsilon_t^2}{((T-t)\Delta\rho)^2} \right) = 1 - \delta_0.
	\end{equation}

	\textit{Bounding Term (II):} Define the event that the induction hypothesis holds for all sampled successor beliefs:
	\begin{equation}
		E \triangleq \bigcap_{i=1}^{N_b} \left\{ |V_{M_P, t+1}^\pi(\bar{b}_{t+1}^i) - \hat{V}_{M_P, t+1}^\pi(\bar{b}_{t+1}^i)| \leq \theta_{t+1} \right\}.
	\end{equation}
	By the union bound and the induction hypothesis:
	\begin{equation}\label{eq:union_bound_proof}
		\begin{aligned}
			P(E) &= 1 - P\left( \bigcup_{i=1}^{N_b} \{|V - \hat{V}| > \theta_{t+1}\} \right) \\
			&\geq 1 - \sum_{i=1}^{N_b} P(|V_{M_P, t+1}^\pi(\bar{b}_{t+1}^i) - \hat{V}_{M_P, t+1}^\pi(\bar{b}_{t+1}^i)| > \theta_{t+1}) \\
			&\geq 1 - N_b \eta_{t+1}.
		\end{aligned}
	\end{equation}
	Conditioned on $E$, we have:
	\begin{equation}
		\begin{aligned}
			\left| \tfrac{1}{N_b}\textstyle\sum_{i=1}^{N_b} (V_{M_P, t+1}^\pi(\bar{b}_{t+1}^i) - \hat{V}_{M_P, t+1}^\pi(\bar{b}_{t+1}^i)) \right|
			&\leq \tfrac{1}{N_b}\textstyle\sum_{i=1}^{N_b} |V_{M_P, t+1}^\pi(\bar{b}_{t+1}^i) - \hat{V}_{M_P, t+1}^\pi(\bar{b}_{t+1}^i)| \\
			&\leq \theta_{t+1}.
		\end{aligned}
	\end{equation}

	Using the decomposition \eqref{eq:error_decomp} and the bounds on terms (I) and (II):
	\begin{equation}
		\begin{aligned}
			P(|Q_{M_P, t}^\pi - \hat{Q}^\pi_{M_P, t}| \leq \epsilon_t + \theta_{t+1})
			&\geq P(\text{(I)} \leq \epsilon_t) + P(\text{(II)} \leq \theta_{t+1}) - 1 \\
			&\geq (1 - \delta_0) + (1 - N_b\eta_{t+1}) - 1 = 1 - \eta_t.
		\end{aligned}
	\end{equation}
	This completes the induction.

	\textbf{Unrolling the Recursions.} Expanding the recursive relations:
	\begin{equation}
		\eta_t = \delta_0 \sum_{k=0}^{T-t-1} N_b^k = \delta_0 \frac{N_b^{T-t} - 1}{N_b - 1}.
	\end{equation}
	\begin{equation}
		\begin{aligned}
			\theta_t &= \sum_{k=0}^{T-t-1} \epsilon_{t+k} = \sum_{k=0}^{T-t-1} (T-t-k)\Delta\rho \sqrt{\frac{\ln(2/\delta_0)}{2N_b}} \\
			&= \Delta\rho \sqrt{\frac{\ln(2/\delta_0)}{2N_b}} \sum_{j=1}^{T-t} j \\
			&= \Delta\rho \cdot S_t \sqrt{\frac{\ln(2/\delta_0)}{2N_b}},
		\end{aligned}
	\end{equation}
	where $S_t = \frac{(T-t)(T-t+1)}{2}$.

	\textbf{Final Bound.} To obtain a bound with confidence $1-\delta$, we set $\eta_t = \delta$, which requires:
	\begin{equation}
		\delta_0 = \frac{\delta(N_b - 1)}{N_b^{T-t} - 1}.
	\end{equation}
	Substituting into $\theta_t$:
	\begin{equation}
		\theta_t = \Delta\rho \cdot S_t \sqrt{\frac{\ln\big(\frac{2(N_b^{T-t}-1)}{\delta(N_b-1)}\big)}{2N_b}}.
	\end{equation}
	Therefore:
	\begin{equation}
		P\left( |Q_{M_P, t}^\pi - \hat{Q}^\pi_{M_P, t}| \leq \Delta\rho \cdot S_t \sqrt{\frac{\ln\big(\frac{2(N_b^{T-t}-1)}{\delta(N_b-1)}\big)}{2N_b}} \right) \geq 1 - \delta.
	\end{equation}
\end{proof}

\section{Proof of Theorem~\ref{thm:cvar_cost_sparse_sampling_guarantees}}\label{proof:cvar_cost_sparse_sampling}

\begin{theorem}[Restated]\label{thm:cvar_cost_sparse_sampling_full}
	Let $\delta \in (0,1)$, let $\bar{b}_t = \{(x_t^i, w_t^i)\}_{i=1}^{N_p}$ be the particle belief at time $t$, and let $|A|$ denote the number of actions. With $\Delta\rho$ and $S_t$ as in Theorem~\ref{thm:cvar_cost_policy_eval_guarantees}, if $N_b > 1$, then
	\begin{equation}
		P \! \left( \left| V_{M_P, t}^*(\bar{b}_t, \alpha) - \hat{V}^*_{M_P, t}(\bar{b}_t, \alpha) \right| \leq \Delta\rho \cdot S_t \sqrt{\frac{\ln\!\big(\frac{2|A|((|A|N_b)^{T-t}-1)}{\delta(|A|N_b-1)}\big)}{2N_b}} \right) \geq 1-\delta.
	\end{equation}
\end{theorem}

\begin{proof}
	The proof follows the inductive structure of Theorem~\ref{thm:cvar_cost_policy_eval_full}. The error decomposition and Hoeffding bound on Term~(I) are identical; the only difference is that computing $V^*_{M_P,t} = \min_{a \in A} Q^*_{M_P,t}(\cdot, a)$ requires the bound to hold \emph{simultaneously for all actions}, introducing union bounds over $|A|$.

	Specifically, the inductive hypothesis becomes $P(\forall a \in A: |Q^* - \hat{Q}^*| \leq \theta_t) \geq 1 - \eta_t$, with modified failure probability recursion:
	\begin{equation}
		\eta_t = |A|\delta_0 + |A|N_b \eta_{t+1}, \quad \eta_T = 0,
	\end{equation}
	where the $|A|$ factors arise from: (i) a union bound over actions for the Hoeffding bound on Term~(I), contributing $|A|\delta_0$; and (ii) a union bound over $|A| \cdot N_b$ sampled successor beliefs (one set of $N_b$ per action) for Term~(II), contributing $|A|N_b\eta_{t+1}$. The error bound $\theta_t = \epsilon_t + \theta_{t+1}$ and $\epsilon_t$ remain unchanged from Theorem~\ref{thm:cvar_cost_policy_eval_full}.

	Unrolling yields $\eta_t = |A|\delta_0 \frac{(|A|N_b)^{T-t} - 1}{|A|N_b - 1}$. Setting $\eta_t = \delta$ gives $\delta_0 = \frac{\delta(|A|N_b - 1)}{|A|((|A|N_b)^{T-t} - 1)}$, and the result follows by substituting into $\theta_t$ and using $|V^* - \hat{V}^*| \leq \max_a |Q^* - \hat{Q}^*|$.
\end{proof}

%%%%%%%%%%%%%%%%%%%%%%%%%%%%%%%%%%%%%%%%%%%%%%%%%%%%%%%%%%%%%%%%%%%%%%%%%%%%%%%
%%%%%%%%%%%%%%%%%%%%%%%%%%%%%%%%%%%%%%%%%%%%%%%%%%%%%%%%%%%%%%%%%%%%%%%%%%%%%%%

\section{CVaR Lipschitz Property}\label{proof:cvar_cost_sensitivity}

\begin{lemma}[CVaR Lipschitz Bound]\label{lem:cvar_cost_sensitivity}
	Let $X$ and $Y$ be random variables supported on $[c_{\min}, c_{\max}]$ with CDF discrepancy $\sup_{z} |F_X(z) - F_Y(z)| \leq \varepsilon$. Then
	\begin{equation}\label{eq:cvar_lipschitz}
		|\text{CVaR}_\alpha(X) - \text{CVaR}_\alpha(Y)| \leq \frac{1}{\alpha}\, W_1(X, Y) \leq \frac{\Delta\rho}{\alpha}\,\varepsilon,
	\end{equation}
	where $W_1(X,Y) = \int_{c_{\min}}^{c_{\max}} |F_X(z) - F_Y(z)|\,dz$ is the Wasserstein-1 distance and $\Delta\rho = c_{\max} - c_{\min}$.
\end{lemma}

\begin{proof}
	$\text{CVaR}_\alpha$ is $1/\alpha$-Lipschitz with respect to $W_1$ \cite{rockafellar2000optimization}, so
	\begin{equation}
		\begin{aligned}
			|\text{CVaR}_\alpha(X) - \text{CVaR}_\alpha(Y)|
			&\leq \frac{W_1(X,Y)}{\alpha}
			= \frac{1}{\alpha}\int_{c_{\min}}^{c_{\max}} |F_X(z) - F_Y(z)|\,dz \\
			&\leq \frac{1}{\alpha}\int_{c_{\min}}^{c_{\max}} \varepsilon\,dz
			= \frac{\varepsilon\,\Delta\rho}{\alpha}.
		\end{aligned}
	\end{equation}
\end{proof}

\section{Weighted DKW Bound}\label{proof:weighted_dkw}

\begin{lemma}[Uniform Concentration of Bounded Monotonic Processes]\label{lem:uniform_monotone}
	Let $\{X_i(z)\}_{i=1}^n$ be $n$ i.i.d.\ random variables such that for all $z \in \mathbb{R}$: (i) $0 \leq X_i(z) \leq M$; and (ii) $X_i(z_a) \leq X_i(z_b)$ for any $z_a < z_b$. Let $F(z) = \mathbb{E}[X_i(z)]$ and $\bar{X}_n(z) = \frac{1}{n} \sum_{i=1}^n X_i(z)$. Then for any $\epsilon > 0$:
	\begin{equation}\label{eq:uniform_monotone}
		P\!\left( \sup_{z \in \mathbb{R}} |\bar{X}_n(z) - F(z)| > 2\epsilon \right) \leq 2 \left( \left\lceil \frac{M}{\epsilon} \right\rceil + 1 \right) \exp\!\left( -\frac{2n\epsilon^2}{M^2} \right).
	\end{equation}
\end{lemma}

\begin{proof}
	Since $F(z)$ is monotonic and bounded by $[0, M]$, choose $k = \lceil M/\epsilon \rceil$ and grid points $z_0 < z_1 < \cdots < z_k$ with $z_0 = -\infty$, $z_k = \infty$, such that
	\begin{equation}
		F(z_j) - F(z_{j-1}) \leq \epsilon \quad \forall\, j \in \{1, \ldots, k\}.
	\end{equation}
	For any $z \in [z_{j-1}, z_j]$, monotonicity of $X_i$ and $F$ gives
	\begin{equation}
		\bar{X}_n(z_{j-1}) \leq \bar{X}_n(z) \leq \bar{X}_n(z_j), \qquad F(z_{j-1}) \leq F(z) \leq F(z_j),
	\end{equation}
	so
	\begin{equation}
		\bar{X}_n(z_{j-1}) - F(z_j) \leq \bar{X}_n(z) - F(z) \leq \bar{X}_n(z_j) - F(z_{j-1}).
	\end{equation}
	Using $F(z_j) - F(z_{j-1}) \leq \epsilon$:
	\begin{align}
		\bar{X}_n(z) - F(z) &\leq \big(\bar{X}_n(z_j) - F(z_j)\big) + \big(F(z_j) - F(z_{j-1})\big) \leq \big(\bar{X}_n(z_j) - F(z_j)\big) + \epsilon, \\
		\bar{X}_n(z) - F(z) &\geq \big(\bar{X}_n(z_{j-1}) - F(z_{j-1})\big) - \big(F(z_j) - F(z_{j-1})\big) \geq \big(\bar{X}_n(z_{j-1}) - F(z_{j-1})\big) - \epsilon.
	\end{align}
	Therefore:
	\begin{equation}\label{eq:sup_to_max_monotone}
		\sup_{z} |\bar{X}_n(z) - F(z)| \leq \max_{j \in \{0, \ldots, k\}} |\bar{X}_n(z_j) - F(z_j)| + \epsilon.
	\end{equation}
	The event $\{\sup_z |\bar{X}_n(z) - F(z)| > 2\epsilon\}$ implies $\{\max_j |\bar{X}_n(z_j) - F(z_j)| > \epsilon\}$. The grid points are deterministic, and at each $z_j$ the summands $X_1(z_j), \ldots, X_n(z_j)$ are i.i.d.\ in $[0, M]$ with mean $F(z_j)$. By the union bound and Hoeffding's inequality:
	\begin{equation}
		P\!\left( \max_j |\bar{X}_n(z_j) - F(z_j)| > \epsilon \right) \leq \sum_{j=0}^{k} P\!\left( |\bar{X}_n(z_j) - F(z_j)| > \epsilon \right) \leq 2(k+1) \exp\!\left( -\frac{2n\epsilon^2}{M^2} \right).
	\end{equation}
	Substituting $k = \lceil M/\epsilon \rceil$ yields \eqref{eq:uniform_monotone}.
\end{proof}

\begin{lemma}[Weighted DKW Bound]\label{lem:weighted_dkw}
	Let $b$ be a belief distribution over $X$ and $\bar{b} = \{(x_i, w_i)\}_{i=1}^{N_p}$ a particle belief with $x_i \stackrel{\text{i.i.d.}}{\sim} q$ and importance weights $w_i = b(x_i)/q(x_i)$, where $d_\infty \triangleq \mathrm{ess\,sup}_{x \sim q}\, b(x)/q(x) \leq d_\infty^{\max} < \infty$. Then for any action $a \in A$ and any $t \in (0, d_\infty^{\max}]$:
	\begin{equation}\label{eq:weighted_dkw}
		\begin{aligned}
			&P\!\left(\sup_{z \in \mathbb{R}} \left| P_{b}(c(x, a) \leq z) - \sum_{i=1}^{N_p} \tilde{w}_i\, \mathbf{1}(c(x_i, a) \leq z) \right| > t \right) \\
			&\qquad \leq 2\!\left(\!\left\lceil \frac{4d_\infty^{\max}}{t} \right\rceil + 2\right)\exp\!\left(-\frac{N_p\, t^2}{8(d_\infty^{\max})^2}\right),
		\end{aligned}
	\end{equation}
	where $\tilde{w}_i = w_i / \sum_{j} w_j$ are the normalized weights.
\end{lemma}

\begin{proof}
	Fix an action $a \in A$. Let $Y_i \triangleq c(x_i, a)$ with $x_i \stackrel{\text{i.i.d.}}{\sim} q$, $F(z) \triangleq P_b(c(x,a) \leq z)$, and $\hat{F}_w(z) \triangleq \sum_{i=1}^{N_p} \tilde{w}_i\, \mathbf{1}(Y_i \leq z)$. We first establish the bound for the unnormalized average, then extend to the self-normalized estimator.

	Write $W \triangleq \sum_{i=1}^{N_p} w_i$ and $S(z) \triangleq \sum_{i=1}^{N_p} w_i\, \mathbf{1}(Y_i \leq z)$, so $\hat{F}_w(z) = S(z)/W$. Since $w_i = b(x_i)/q(x_i)$ and $x_i \sim q$, a change of measure gives $\mathbb{E}_q[w_i\,\mathbf{1}(Y_i \leq z)] = F(z)$ and $\mathbb{E}_q[w_i] = 1$. The random variables $X_i(z) \triangleq w_i\,\mathbf{1}(Y_i \leq z)$ are i.i.d., bounded in $[0, d_\infty^{\max}]$, and monotonically non-decreasing in $z$. By \Cref{lem:uniform_monotone} with $M = d_\infty^{\max}$, $n = N_p$, and $\epsilon = t/2$:
	\begin{equation}\label{eq:unnorm_dkw}
		P\!\left(\sup_{z} \left|\frac{S(z)}{N_p} - F(z)\right| > t\right) \leq 2\!\left(\!\left\lceil \frac{2d_\infty^{\max}}{t} \right\rceil + 1\right)\exp\!\left(-\frac{N_p\,t^2}{2(d_\infty^{\max})^2}\right).
	\end{equation}

	It remains to pass from the unnormalized average $S(z)/N_p$ to the self-normalized estimator $\hat{F}_w(z) = S(z)/W$. Adding and subtracting $S(z)/N_p$:
	\begin{equation}\label{eq:self_norm_decomp}
		\hat{F}_w(z) - F(z) = \underbrace{\frac{S(z)}{W} - \frac{S(z)}{N_p}}_{\text{(A)}} + \underbrace{\frac{S(z)}{N_p} - F(z)}_{\text{(B)}}.
	\end{equation}
	Term (B) is bounded uniformly over $z$ by \eqref{eq:unnorm_dkw}. For term (A):
	\begin{equation}\label{eq:term_A_bound}
		\left|\frac{S(z)}{W} - \frac{S(z)}{N_p}\right| = \frac{S(z)}{W} \cdot \left|1 - \frac{W}{N_p}\right| \leq \left|1 - \frac{W}{N_p}\right|,
	\end{equation}
	where we used $S(z) \leq \sum_{i=1}^{N_p} w_i = W$ (since $\mathbf{1}(Y_i \leq z) \leq 1$), so $S(z)/W \leq 1$.

	Since $w_i \in [0, d_\infty^{\max}]$ are i.i.d.\ with $\mathbb{E}[w_i] = 1$, Hoeffding's inequality gives
	\begin{equation}\label{eq:weight_concentration}
		P\!\left(\left|\frac{W}{N_p} - 1\right| > t\right) \leq 2\exp\!\left(-\frac{2N_p\,t^2}{(d_\infty^{\max})^2}\right).
	\end{equation}

	Define the events $E_1 \triangleq \{\sup_z |S(z)/N_p - F(z)| \leq t\}$ and $E_2 \triangleq \{|W/N_p - 1| \leq t\}$. On $E_1 \cap E_2$, by \eqref{eq:self_norm_decomp}, \eqref{eq:term_A_bound}, and the triangle inequality:
	\begin{equation}
		\sup_z |\hat{F}_w(z) - F(z)| \leq \sup_z |\text{(A)}| + \sup_z |\text{(B)}| \leq t + t = 2t.
	\end{equation}
	Therefore $\{\sup_z |\hat{F}_w(z) - F(z)| > 2t\} \subseteq E_1^c \cup E_2^c$, and by the union bound, \eqref{eq:unnorm_dkw}, and \eqref{eq:weight_concentration}:
	\begin{equation}
		P\!\left(\sup_z |\hat{F}_w(z) - F(z)| > 2t\right) \leq 2\!\left(\!\left\lceil \frac{2d_\infty^{\max}}{t} \right\rceil + 1\right)\exp\!\left(-\frac{N_p\,t^2}{2(d_\infty^{\max})^2}\right) + 2\exp\!\left(-\frac{2N_p\,t^2}{(d_\infty^{\max})^2}\right).
	\end{equation}
	The second term has a larger exponent, so the first dominates. Substituting $t' = 2t$:
	\begin{equation}
		P\!\left(\sup_z |\hat{F}_w(z) - F(z)| > t'\right) \leq 2\!\left(\!\left\lceil \frac{4d_\infty^{\max}}{t'} \right\rceil + 2\right)\exp\!\left(-\frac{N_p\,t'^2}{8(d_\infty^{\max})^2}\right).
	\end{equation}
\end{proof}

\section{Proof of Theorem~\ref{thm:concrete_approximation}}\label{proof:concrete_approximation}

\subsection{State Trajectory Coupling}\label{sec:ss_omega_framework}

Define $\omega = (x_0, z_1, x_1, z_2, \ldots, x_T) \in (X \times Z)^T \times X$ to be a \emph{state trajectory}: a realization of the true hidden states and observations over the planning horizon. Both the true belief $b_\ell(\omega)$ and the particle belief $\bar{b}_\ell(\omega)$ are deterministic functions of the same trajectory $\omega$:
\begin{itemize}
	\item $b_\ell(\omega)$ is the Bayesian posterior given observations $z_{1:\ell}$, computed via \eqref{eq:belief_update};
	\item $\bar{b}_\ell(\omega) = \{(x_i^{(\ell)}, w_i^{(\ell)})\}_{i=1}^{N_p}$ is the particle belief with $x_i^{(\ell)} \stackrel{\text{i.i.d.}}{\sim} q_\ell$ and importance weights $w_i^{(\ell)} = \prod_{n=1}^{\ell} O(z_n \mid x_i^{(n)})$, as defined in \Cref{sec:pb_mdp}.
\end{itemize}
Because the particles are i.i.d.\ from $q_\ell$ at each depth $\ell$, and both $b_\ell(\omega)$ and $\bar{b}_\ell(\omega)$ are functions of the same $\omega$, the Q-functions $Q_{M,t}^\pi = \mathbb{E}_\omega[R_M(\omega)]$ and $Q_{M_P,t}^\pi = \mathbb{E}_\omega[R_{M_P}(\omega)]$ are expectations over the same measure $P_\omega$.

\subsection{Proof of \Cref{prop:cvar_cost_approximation}}\label{sec:coupling_bound}

The proof uses the state trajectory coupling from \Cref{sec:ss_omega_framework}.

\begin{proof}
	Fix $\omega \in \mathcal{G}$. Define the per-trajectory returns:
	\begin{equation}
		R_M(\omega) \triangleq \sum_{\ell=t}^{T} \gamma^{\ell-t}\, c_\alpha(b_\ell(\omega), a_\ell), \qquad
		R_{M_P}(\omega) \triangleq \sum_{\ell=t}^{T} \gamma^{\ell-t}\, \rho_\alpha(\bar{b}_\ell(\omega), a_\ell),
	\end{equation}
	where $a_t = a$ and $a_\ell = \pi(h_\ell)$ for $\ell > t$. By the triangle inequality:
	\begin{equation}\label{eq:pointwise_triangle_app}
		\left| R_M(\omega) - R_{M_P}(\omega) \right|
		\leq \sum_{\ell=t}^{T} \gamma^{\ell-t}\, \left| c_\alpha(b_\ell(\omega), a_\ell) - \rho_\alpha(\bar{b}_\ell(\omega), a_\ell) \right|.
	\end{equation}

	At each step $\ell$, $c_\alpha(b_\ell(\omega), a_\ell)$ and $\rho_\alpha(\bar{b}_\ell(\omega), a_\ell)$ are $\text{CVaR}_\alpha$ of the cost $c(x, a_\ell)$ under $x \sim b_\ell(\omega)$ and $x \sim \bar{b}_\ell(\omega)$, respectively. Since $\omega \in \mathcal{G}$, the CDF discrepancy between these distributions is at most $\varepsilon(\ell, N_p) < \alpha$. By \Cref{lem:cvar_cost_sensitivity}:
	\begin{equation}
		\left| c_\alpha(b_\ell(\omega), a_\ell) - \rho_\alpha(\bar{b}_\ell(\omega), a_\ell) \right| \leq \frac{\Delta\rho}{\alpha}\, \varepsilon(\ell, N_p).
	\end{equation}

	Substituting into \eqref{eq:pointwise_triangle_app}:
	\begin{equation}
		\left| R_M(\omega) - R_{M_P}(\omega) \right| \leq \frac{\Delta\rho}{\alpha} \sum_{\ell=t}^{T} \gamma^{\ell-t}\, \varepsilon(\ell, N_p)
	\end{equation}
	for every $\omega \in \mathcal{G}$. The right-hand side does not depend on $\omega$, so taking expectations preserves the bound:
	\begin{equation}
		\begin{aligned}
			\left| Q_{M,t}^\pi - Q_{M_P,t}^\pi \right|
			&= \left| \mathbb{E}_\omega[R_M(\omega)] - \mathbb{E}_\omega[R_{M_P}(\omega)] \right|
			\leq \mathbb{E}_\omega\!\left[ |R_M(\omega) - R_{M_P}(\omega)| \right] \\
			&\leq \frac{\Delta\rho}{\alpha} \sum_{\ell=t}^{T} \gamma^{\ell-t}\, \varepsilon(\ell, N_p),
		\end{aligned}
	\end{equation}
	with probability at least $P(\mathcal{G}) \geq 1 - \delta_0$.
\end{proof}

\subsection{Completing the Proof of Theorem~\ref{thm:concrete_approximation}}

Under the state trajectory coupling (\Cref{sec:ss_omega_framework}), the particles $x_i^{(\ell)} \stackrel{\text{i.i.d.}}{\sim} q_\ell$ are independent of $\omega$, while the importance weights $w_i^{(\ell)} = \prod_{n=1}^{\ell} O(z_n \mid x_i^{(n)})$ depend on $\omega$ through the observations. For a fixed policy $\pi$ and query action $a$, the action at each step is determined: $a_t = a$ and $a_\ell = \pi(b_\ell)$ for $\ell > t$, so the CDF discrepancy only needs to hold for one action per step.

Fix $\omega$ and $\ell \in \{t, \ldots, T\}$. Conditional on $\omega$, the belief $b_\ell(\omega)$ is a fixed distribution and $d_\infty(b_\ell(\omega) \| q_\ell) \leq d_\infty^{\max}$ (condition (v) of \citet{lim2023optimality}). Applying \Cref{lem:weighted_dkw} for the single action $a_\ell$:
\begin{equation}\label{eq:dkw_conditional}
	\begin{aligned}
		&P\!\left(\sup_{z} \left| P_{b_\ell(\omega)}(c(x, a_\ell) \leq z) - \sum_{i=1}^{N_p} \tilde{w}_i^{(\ell)}\, \mathbf{1}(c(x_i^{(\ell)}, a_\ell) \leq z) \right| > t \;\middle|\; \omega \right) \\
		&\qquad \leq 2\!\left(\!\left\lceil \frac{4d_\infty^{\max}}{t} \right\rceil + 2\right)\exp\!\left(-\frac{N_p\, t^2}{8(d_\infty^{\max})^2}\right).
	\end{aligned}
\end{equation}
Since the right-hand side does not depend on $\omega$, marginalizing over $\omega$ preserves the bound:
\begin{equation}\label{eq:dkw_unconditional}
	P\!\left(\varepsilon(\ell, N_p, \omega) > t \right)
	\leq 2\!\left(\!\left\lceil \frac{4d_\infty^{\max}}{t} \right\rceil + 2\right)\exp\!\left(-\frac{N_p\, t^2}{8(d_\infty^{\max})^2}\right).
\end{equation}
Taking a union bound over the $D+1$ time steps $\ell \in \{t, \ldots, T\}$ and setting the right-hand side equal to $\delta_0/(D+1)$, we solve for $t$. Denoting $C(t) \triangleq 2(\lceil 4d_\infty^{\max}/t \rceil + 2)$, the condition $C(t)(D+1)\exp(-N_p t^2/(8(d_\infty^{\max})^2)) \leq \delta_0$ is satisfied by
\begin{equation}
	\varepsilon(\ell, N_p) \leq d_\infty^{\max}\sqrt{\frac{8\ln\!\left(\frac{C(\varepsilon)(D+1)}{\delta_0}\right)}{N_p}} \quad \text{for all } \ell \in \{t, \ldots, T\},
\end{equation}
with probability at least $1 - \delta_0$, where $C(\varepsilon) = 2(\lceil 4d_\infty^{\max}/\varepsilon \rceil + 2)$. For $\varepsilon \leq d_\infty^{\max}$, $C(\varepsilon) \leq 12d_\infty^{\max}/\varepsilon$, so the bound is implicitly determined but scales as $O(d_\infty^{\max}\sqrt{\ln(d_\infty^{\max}(D+1)/(\varepsilon\,\delta_0))/N_p})$. In particular, the condition $N_p = \tilde{\Omega}((d_\infty^{\max}/\alpha)^2)$ ensures $\varepsilon(\ell, N_p) < \alpha$, as required by \Cref{prop:cvar_cost_approximation}. Substituting the uniform bound into \eqref{eq:approximation_bound}:
\begin{equation}
	\begin{aligned}
	|Q_{M,t}^\pi(b_t, a, \alpha) - Q_{M_P,t}^\pi(\bar{b}_t, a, \alpha)|
	&\leq \frac{\Delta\rho}{\alpha} \sum_{\ell=t}^{T} \gamma^{\ell-t}\, \varepsilon(\ell, N_p)
	\leq \frac{\Delta\rho\, d_\infty^{\max}}{\alpha} \cdot G_{D+1} \cdot \varepsilon,
	\end{aligned}
\end{equation}
where $\varepsilon$ satisfies the implicit bound above and $G_{D+1} = (1 - \gamma^{D+1})/(1-\gamma)$. \qed

\section{Alternative Approximation Bound via Induction}\label{sec:alternative_approx}

We present an alternative bound that explicitly separates the two sources of approximation error---CVaR cost discrepancy ($\varepsilon_C$) and observation distribution mismatch ($\varepsilon_{TV}$)---using backward induction rather than the coupling-based approach of \Cref{prop:cvar_cost_approximation}.

Assume that for each time step $\ell \in \{t, \ldots, T\}$ and all actions $a \in A$, the coupled belief pair $(b_\ell, \bar{b}_\ell)$ satisfies, with probability at least $1 - 2\delta_0$:
\begin{enumerate}
	\item \textbf{Cost CDF discrepancy:} $\displaystyle\sup_{z} \left|P_{b_\ell}(c(x, a) \leq z) - \textstyle\sum_{i=1}^{N_p} \tilde{w}_i \mathbf{1}(c(x_i, a) \leq z)\right| \leq \varepsilon_C$.
	\item \textbf{Observation TV distance:} $d_{TV}\!\left(P_o(\cdot \mid b_\ell, a),\; P_o^{PB}(\cdot \mid \bar{b}_\ell, a)\right) \leq \varepsilon_{TV}$.
\end{enumerate}

\begin{proposition}[Two-Source Approximation Bound]\label{prop:two_source_approx}
	Under the assumptions above with $\varepsilon_C < \alpha$, for any policy $\pi$, realizable belief $b_t$, corresponding particle belief $\bar{b}_t$ with $N_p$ particles, and all actions $a$:
	\begin{equation}\label{eq:two_source_bound}
		|Q_{M,t}^\pi(b_t, a, \alpha) - Q_{M_P,t}^\pi(\bar{b}_t, a, \alpha)| \leq \frac{\Delta\rho}{\alpha}\,\varepsilon_C \cdot G_{D+1} + \gamma\,V_{\max}\,\varepsilon_{TV} \cdot G_D,
	\end{equation}
	with probability at least $1 - 2(T - t + 1)\delta_0$, where $D = T - t$, $V_{\max} = \Delta\rho\,G_T$, and $G_n = \frac{1 - \gamma^n}{1 - \gamma}$.
\end{proposition}

\begin{proof}
	The proof proceeds by backward induction on $\ell$, following \citet{lim2023optimality} and \citet{pariente2026acceleratedonlineriskaversepolicy}. Define $\lambda_\alpha \triangleq \frac{\Delta\rho}{\alpha}\varepsilon_C$, $\mu \triangleq \gamma V_{\max} \varepsilon_{TV}$, and the ``good'' event $\mathcal{G} \triangleq \bigcap_{\ell=t}^{T} \mathcal{G}_\ell$, where $\mathcal{G}_\ell$ is the event that conditions (1)--(2) hold at time $\ell$.

	\textbf{Inductive Hypothesis.} Conditioned on $\mathcal{G}$, for all $\ell \in \{t, \ldots, T\}$:
	\begin{equation}\label{eq:two_source_ih}
		|Q_{M,\ell}^\pi - Q_{M_P,\ell}^\pi| \leq \beta_\ell, \quad \text{where } \beta_\ell = \lambda_\alpha + \mu + \gamma\,\beta_{\ell+1}, \quad \beta_T = \lambda_\alpha.
	\end{equation}

	\textbf{Base Case ($\ell = T$).} The Q-functions reduce to immediate CVaR costs. By \Cref{lem:cvar_cost_sensitivity} and condition (1):
	\begin{equation}
		|Q_{M,T}^\pi - Q_{M_P,T}^\pi| = |c_\alpha(b_T, a) - \rho_\alpha(\bar{b}_T, a)| \leq \lambda_\alpha = \beta_T.
	\end{equation}

	\textbf{Inductive Step.} Assume the bound holds at $\ell+1$. From the Q-function definitions:
	\begin{equation}\label{eq:two_source_decomp}
		\begin{aligned}
			|Q_{M,\ell}^\pi - Q_{M_P,\ell}^\pi|
			&= \big| \big(c_\alpha(b_\ell, a_\ell) + \gamma\,\mathbb{E}_{M}[V_{M,\ell+1}^\pi]\big) - \big(\rho_\alpha(\bar{b}_\ell, a_\ell) + \gamma\,\mathbb{E}_{M_P}[V_{M_P,\ell+1}^\pi]\big) \big| \\
			&\leq \underbrace{|c_\alpha(b_\ell, a_\ell) - \rho_\alpha(\bar{b}_\ell, a_\ell)|}_{\text{(I)}} + \gamma\,\underbrace{|\mathbb{E}_{M}[V_{M,\ell+1}^\pi] - \mathbb{E}_{M_P}[V_{M_P,\ell+1}^\pi]|}_{\text{(II)}}.
		\end{aligned}
	\end{equation}
	Term (I) is bounded by $\lambda_\alpha$ via \Cref{lem:cvar_cost_sensitivity} and condition (1). For term (II), add and subtract $\mathbb{E}_M[V_{M_P,\ell+1}^\pi]$:
	\begin{equation}
		\text{(II)} \leq \underbrace{\mathbb{E}_{M}\!\left[|V_{M,\ell+1}^\pi - V_{M_P,\ell+1}^\pi|\right]}_{\text{(IIa): recursive discrepancy}} + \underbrace{|\mathbb{E}_{M}[V_{M_P,\ell+1}^\pi] - \mathbb{E}_{M_P}[V_{M_P,\ell+1}^\pi]|}_{\text{(IIb): observation distribution shift}}.
	\end{equation}
	Term (IIa) is bounded by $\beta_{\ell+1}$ via the induction hypothesis. For term (IIb), since $V_{M_P,\ell+1}^\pi$ has range $[0, V_{\max}]$ and the observation distributions differ by $d_{TV} \leq \varepsilon_{TV}$ (condition 2):
	\begin{equation}
		|\mathbb{E}_{M}[V_{M_P,\ell+1}^\pi] - \mathbb{E}_{M_P}[V_{M_P,\ell+1}^\pi]| \leq V_{\max}\,\varepsilon_{TV}.
	\end{equation}
	Combining:
	\begin{equation}
		|Q_{M,\ell}^\pi - Q_{M_P,\ell}^\pi| \leq \lambda_\alpha + \gamma(V_{\max}\,\varepsilon_{TV} + \beta_{\ell+1}) = \lambda_\alpha + \mu + \gamma\,\beta_{\ell+1} = \beta_\ell.
	\end{equation}

	\textbf{Unrolling.} The recursion $\beta_\ell = \lambda_\alpha + \mu + \gamma\,\beta_{\ell+1}$ with $\beta_T = \lambda_\alpha$ gives:
	\begin{equation}
		\beta_t = \lambda_\alpha \sum_{k=0}^{D} \gamma^k + \mu \sum_{k=0}^{D-1} \gamma^k = \lambda_\alpha\,G_{D+1} + \mu\,G_D.
	\end{equation}
	By union bound, $P(\mathcal{G}^c) \leq 2(T-t+1)\delta_0$.
\end{proof}

\section{Extended Results}\label{sec:extended_results}

\Cref{tab:extended_results} presents the full experimental results including the expected return $\mathbb{E}[G]$, static CVaR of the cumulative return $\text{CVaR}_{0.1}[G]$, CVaR cost return, dangerous encounters, and goal rate. The expected return and static CVaR use the standard (non-CVaR) immediate cost, i.e., the return $G = \sum_{k=0}^{T-1}\gamma^k c(b_k, a_k)$ where $c(b_k,a_k) = \mathbb{E}_{x \sim b_k}[c(x,a_k)]$, with CVaR computed over the distribution of episode returns. All results are over 500 episodes.

\begin{table}[h]
	\centering
	\caption{Extended results across 500 episodes with $\alpha = 0.1$, $\gamma = 0.95$. $\mathbb{E}[G]$: expected cumulative return (mean $\pm$ 95\% CI). $\text{CVaR}_{0.1}[G]$: static CVaR of the cumulative return distribution (point estimate; confidence intervals omitted due to high variance of this tail statistic). CVaR Cost Return, Danger Encounters, and Goal Rate are as in \Cref{tab:cvar_cost_results}.}\label{tab:extended_results}
	\small
	\setlength{\tabcolsep}{3.5pt}
	\begin{tabular}{llccccc}
		\toprule
		Environment & Planner & $\mathbb{E}[G]$ & $\text{CVaR}_{0.1}[G]$ & CVaR Cost Return & Danger & Goal \\
		\midrule
		\multirow{6}{*}{LightDark}
		& PFT-DPW            & $-96.1 \pm 15.3$ & $-463.4$ & $777.7 \pm 26.8$ & $2.16 \pm 0.09$ & $100\%$ \\
		& ICVaR-PFT-DPW      & $-49.1 \pm 9.6$ & $-290.5$ & $250.5 \pm 24.1$ & $0.97 \pm 0.11$ & $100\%$ \\
		& CVaRCost-PFT-DPW   & $\mathbf{-30.1 \pm 6.5}$ & $\mathbf{-186.1}$ & $\mathbf{144.8 \pm 11.1}$ & $\mathbf{0.61 \pm 0.07}$ & $100\%$ \\
		& POMCPOW            & $-104.1 \pm 15.3$ & $-451.9$ & $820.3 \pm 26.7$ & $2.28 \pm 0.09$ & $100\%$ \\
		& ICVaR-POMCPOW      & $-21.1 \pm 3.5$ & $-76.6$ & $58.6 \pm 6.6$ & $0.31 \pm 0.06$ & $100\%$ \\
		& CVaRCost-POMCPOW   & $\mathbf{-23.3 \pm 3.6}$ & $\mathbf{-86.2}$ & $\mathbf{87.9 \pm 6.5}$ & $\mathbf{0.45 \pm 0.06}$ & $100\%$ \\
		\midrule
		\multirow{6}{*}{PushPOMDP}
		& PFT-DPW            & $-13.6 \pm 1.3$ & $-51.8$ & $47.4 \pm 3.0$ & $2.83 \pm 0.17$ & $68 \pm 4\%$ \\
		& ICVaR-PFT-DPW      & $-17.6 \pm 1.6$ & $-62.5$ & $76.8 \pm 4.4$ & $3.60 \pm 0.22$ & $59 \pm 4\%$ \\
		& CVaRCost-PFT-DPW   & $\mathbf{-15.2 \pm 0.7}$ & $\mathbf{-33.3}$ & $\mathbf{21.0 \pm 0.8}$ & $\mathbf{1.78 \pm 0.11}$ & $27 \pm 4\%$ \\
		& POMCPOW            & $-12.2 \pm 1.3$ & $-44.9$ & $31.2 \pm 2.3$ & $2.79 \pm 0.19$ & $64 \pm 4\%$ \\
		& ICVaR-POMCPOW      & $-12.5 \pm 0.9$ & $-33.1$ & $22.7 \pm 1.3$ & $2.18 \pm 0.14$ & $71 \pm 4\%$ \\
		& CVaRCost-POMCPOW   & $\mathbf{-13.9 \pm 0.5}$ & $\mathbf{-22.4}$ & $\mathbf{17.7 \pm 1.0}$ & $\mathbf{1.66 \pm 0.11}$ & $33 \pm 4\%$ \\
		\midrule
		\multirow{6}{*}{Cont.\ LaserTag}
		& PFT-DPW            & $-13.3 \pm 2.2$ & $-67.3$ & $119.9 \pm 11.6$ & $2.38 \pm 0.26$ & $100\%$ \\
		& ICVaR-PFT-DPW      & $-6.2 \pm 2.0$ & $-60.8$ & $77.6 \pm 9.1$ & $1.59 \pm 0.20$ & $100\%$ \\
		& CVaRCost-PFT-DPW   & $-6.3 \pm 0.7$ & $\mathbf{-20.3}$ & $\mathbf{12.9 \pm 1.4}$ & $\mathbf{0.14 \pm 0.04}$ & $100\%$ \\
		& POMCPOW            & $\mathbf{-1.0 \pm 3.5}$ & $-83.2$ & $188.0 \pm 8.0$ & $4.03 \pm 0.14$ & $100\%$ \\
		& ICVaR-POMCPOW      & $-5.5 \pm 0.3$ & $-11.4$ & $8.1 \pm 0.6$ & $0.02 \pm 0.02$ & $100\%$ \\
		& CVaRCost-POMCPOW   & $\mathbf{-4.3 \pm 0.5}$ & $\mathbf{-11.5}$ & $\mathbf{5.4 \pm 0.5}$ & $\mathbf{0.04 \pm 0.03}$ & $100\%$ \\
		\midrule
		\multirow{6}{*}{PacMan}
		& PFT-DPW            & $-5.2 \pm 1.3$ & $-43.0$ & $57.3 \pm 3.1$ & $0.73 \pm 0.04$ & --- \\
		& ICVaR-PFT-DPW      & $-7.7 \pm 1.4$ & $-48.9$ & $57.5 \pm 3.1$ & $0.76 \pm 0.04$ & --- \\
		& CVaRCost-PFT-DPW   & $-14.3 \pm 1.0$ & $\mathbf{-42.7}$ & $\mathbf{39.5 \pm 1.8}$ & $\mathbf{0.53 \pm 0.04}$ & --- \\
		& POMCPOW            & $-9.6 \pm 2.0$ & $-67.6$ & $95.6 \pm 5.0$ & $0.90 \pm 0.03$ & --- \\
		& ICVaR-POMCPOW      & $\mathbf{-6.3 \pm 1.3}$ & $-41.6$ & $43.4 \pm 2.0$ & $0.63 \pm 0.04$ & --- \\
		& CVaRCost-POMCPOW   & $-14.5 \pm 1.1$ & $\mathbf{-45.5}$ & $\mathbf{37.7 \pm 1.5}$ & $\mathbf{0.57 \pm 0.04}$ & --- \\
		\bottomrule
	\end{tabular}
\end{table}

The extended results reveal several patterns. First, CVaR cost planners consistently achieve the best (least negative) static CVaR of the cumulative return across all environments, even though they optimize the CVaR cost return rather than the static CVaR directly. This suggests that per-step tail risk reduction also reduces trajectory-level tail risk. Second, the expected return $\mathbb{E}[G]$ shows a nuanced picture: in LightDark, CVaR cost planners achieve better expected return (e.g., CVaRCost-POMCPOW $-23.3$ vs.\ POMCPOW $-104.1$), indicating that avoiding dangerous states also improves expected performance. In ContinuousLaserTag, CVaRCost-PFT-DPW has slightly worse expected return ($-6.3$ vs.\ $-1.0$ for POMCPOW) because it avoids the zero-expected-value danger zones that occasionally yield positive rewards, but achieves far better tail risk ($-20.3$ vs.\ $-83.2$ static CVaR). In PacMan, CVaR cost planners trade slightly worse expected return for substantially better tail risk and fewer ghost collisions.

\section{Ablation Studies}\label{sec:ablation_studies}

We study the sensitivity of CVaR cost planning to two key parameters: the risk level $\alpha$ and the particle count $N_p$. All ablation experiments use PFT-DPW as the base planner on LightDark and ContinuousLaserTag, with 300 episodes per configuration and the same parameters as in \Cref{sec:experiments} unless otherwise noted.

\subsection{Sensitivity to Risk Level $\alpha$}

\Cref{fig:alpha_ablation} shows the CVaR cost return as $\alpha$ varies over $\{0.01, 0.05, 0.1, 0.25, 0.5, 1.0\}$ with $N_p = 200$ particles.

\begin{figure}[h]
	\centering
	\includegraphics[width=\textwidth]{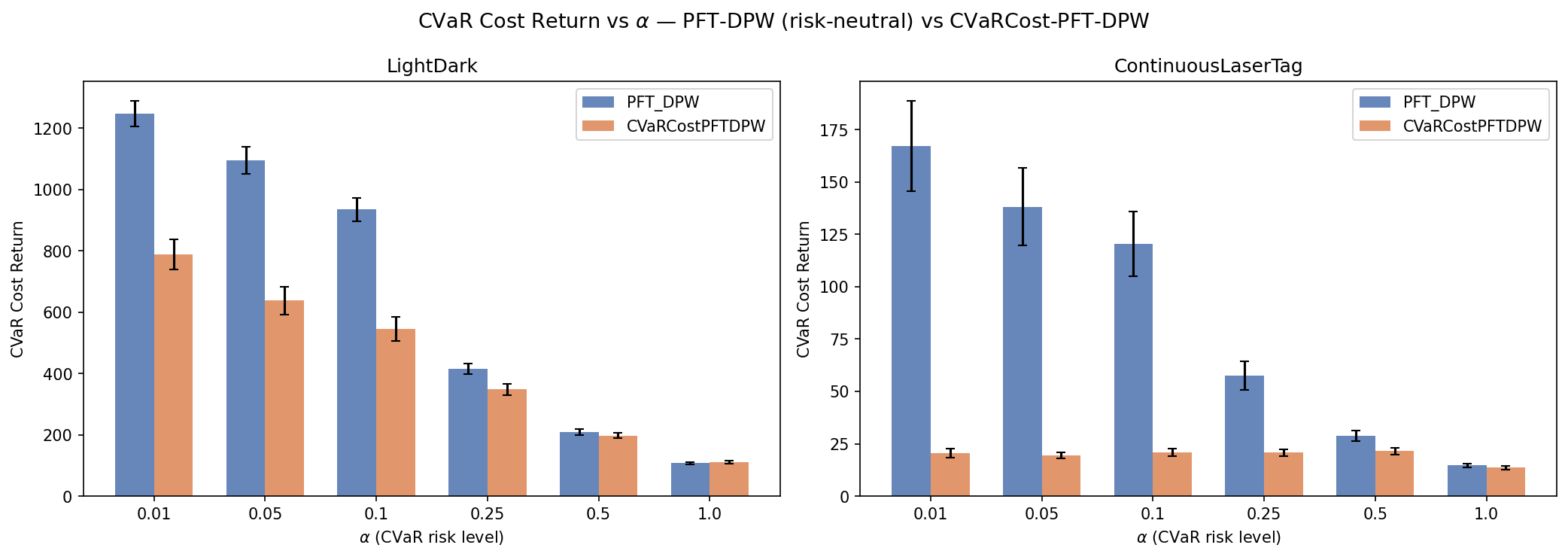}
	\caption{CVaR cost return vs.\ risk level $\alpha$ for PFT-DPW (risk-neutral) and CVaRCost-PFT-DPW on LightDark (left) and ContinuousLaserTag (right). Error bars show 95\% confidence intervals over 300 episodes. At $\alpha = 1$, both planners converge to the same performance, confirming that CVaR cost recovers standard expectation-based planning.}
	\label{fig:alpha_ablation}
\end{figure}

Two observations stand out. First, at $\alpha = 1$ the CVaR cost planner recovers risk-neutral performance: in LightDark, CVaRCost-PFT-DPW ($109.9$) matches PFT-DPW ($107.3$) with overlapping confidence intervals, and similarly in ContinuousLaserTag ($13.6$ vs.\ $14.6$). This confirms the theoretical property that CVaR cost reduces to expected cost when $\alpha = 1$. Second, as $\alpha$ decreases the gap between CVaRCost-PFT-DPW and PFT-DPW widens, with CVaRCost-PFT-DPW consistently achieving lower CVaR cost return. In ContinuousLaserTag, CVaRCost-PFT-DPW maintains a nearly flat CVaR cost return ($\approx 20$) across all $\alpha \leq 0.5$, while PFT-DPW degrades sharply from $14.6$ at $\alpha = 1$ to $167.1$ at $\alpha = 0.01$. This demonstrates that CVaR cost planning effectively suppresses tail risk regardless of the specific $\alpha$ value, while the risk-neutral planner's CVaR cost return increases as the metric focuses on more extreme tails.

\subsection{Sensitivity to Particle Count $N_p$}

\Cref{fig:particle_ablation} shows the CVaR cost return as $N_p$ varies over $\{50, 100, 200, 500\}$ with $\alpha = 0.1$.

\begin{figure}[h]
	\centering
	\includegraphics[width=\textwidth]{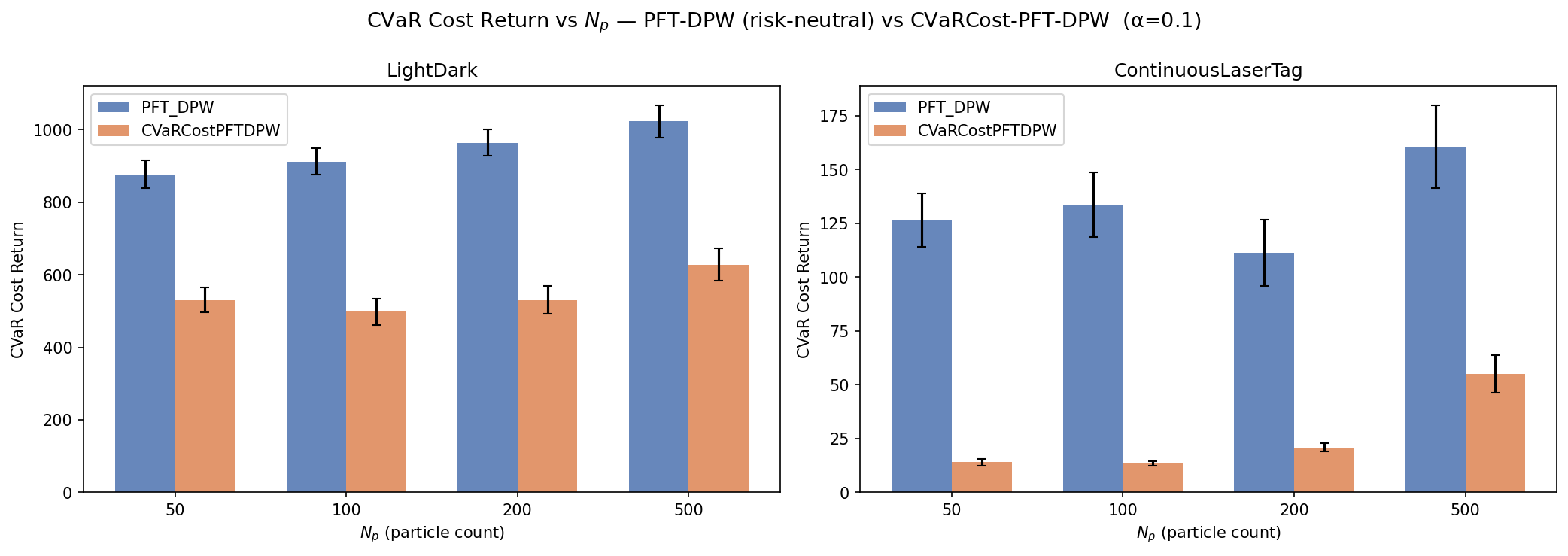}
	\caption{CVaR cost return vs.\ particle count $N_p$ for PFT-DPW (risk-neutral) and CVaRCost-PFT-DPW on LightDark (left) and ContinuousLaserTag (right). Error bars show 95\% confidence intervals over 300 episodes. Fixed $\alpha = 0.1$.}
	\label{fig:particle_ablation}
\end{figure}

The CVaR cost advantage over the risk-neutral baseline is robust across all particle counts. In LightDark, CVaRCost-PFT-DPW maintains a consistent advantage over PFT-DPW at every $N_p$ (e.g., $498.0$ vs.\ $912.6$ at $N_p = 100$). In ContinuousLaserTag, CVaRCost-PFT-DPW achieves substantially lower CVaR cost return than PFT-DPW across all particle counts, though both planners' CVaR cost return increases at $N_p = 500$, likely due to the computational overhead of maintaining more particles within the fixed 4-second planning timeout, which reduces the effective search depth.

\section{Hyperparameter Tuning}\label{sec:hyperparameter_tuning}

All planners (PFT-DPW, CVaRCost-PFT-DPW, ICVaR-PFT-DPW, POMCPOW, CVaRCost-POMCPOW, ICVaR-POMCPOW) were tuned independently per environment using Optuna (50 trials, 10 episodes per trial), minimizing the CVaR cost return for all planners to ensure a fair comparison. The tuned parameters are the search tree depth, UCB exploration constant, and progressive widening coefficients ($k_a$, $\alpha_a$, $k_o$, $\alpha_o$). For each (environment, planner) pair, we selected the better result between the tuned and default parameters. \Cref{tab:tuned_params} reports the final hyperparameters used in all experiments.

\begin{table}[h]
	\centering
	\caption{Tuned hyperparameters per (environment, planner) pair. All planners use $\gamma = 0.95$, $\alpha = 0.1$, $N_p = 200$ particles, and 4-second planning timeout.}\label{tab:tuned_params}
	\small
	\setlength{\tabcolsep}{3pt}
	\begin{tabular}{llrrrrrr}
		\toprule
		Environment & Planner & Depth & Expl.\ Const. & $k_a$ & $\alpha_a$ & $k_o$ & $\alpha_o$ \\
		\midrule
		\multirow{6}{*}{LightDark}
		& PFT-DPW          & 2  & 9534.9  & 1  & 0.28 & 6  & 0.49 \\
		& CVaRCost-PFT-DPW & 2  & 650.3   & 10 & 0.50 & 10 & 0.22 \\
		& ICVaR-PFT-DPW    & 2  & 9657.0  & 6  & 0.01 & 1  & 0.19 \\
		& POMCPOW          & 15 & 8000.9  & 6  & 0.35 & 6  & 0.36 \\
		& CVaRCost-POMCPOW & 2  & 11595.5 & 4  & 0.48 & 7  & 0.46 \\
		& ICVaR-POMCPOW    & 3  & 5659.7  & 10 & 0.39 & 1  & 0.19 \\
		\midrule
		\multirow{6}{*}{PushPOMDP}
		& PFT-DPW          & 15 & 403.2   & 6  & 0.30 & 9  & 0.49 \\
		& CVaRCost-PFT-DPW & 4  & 4108.1  & 9  & 0.23 & 7  & 0.45 \\
		& ICVaR-PFT-DPW    & 5  & 2512.5  & 6  & 0.17 & 6  & 0.40 \\
		& POMCPOW          & 10 & 3727.5  & 5  & 0.40 & 2  & 0.21 \\
		& CVaRCost-POMCPOW & 2  & 2585.7  & 7  & 0.20 & 1  & 0.28 \\
		& ICVaR-POMCPOW    & 2  & 2127.3  & 5  & 0.33 & 2  & 0.16 \\
		\midrule
		\multirow{6}{*}{Cont.\ LaserTag}
		& PFT-DPW          & 6  & 931.9   & 1  & 0.13 & 1  & 0.02 \\
		& CVaRCost-PFT-DPW & 2  & 240.0   & 3  & 0.11 & 2  & 0.24 \\
		& ICVaR-PFT-DPW    & 2  & 568.1   & 1  & 0.50 & 5  & 0.14 \\
		& POMCPOW          & 11 & 192.3   & 1  & 0.30 & 4  & 0.10 \\
		& CVaRCost-POMCPOW & 3  & 613.4   & 7  & 0.38 & 10 & 0.17 \\
		& ICVaR-POMCPOW    & 4  & 238.4   & 5  & 0.32 & 8  & 0.41 \\
		\midrule
		\multirow{6}{*}{PacMan}
		& PFT-DPW          & 2  & 5461.6  & 7  & 0.33 & 5  & 0.42 \\
		& CVaRCost-PFT-DPW & 5  & 6747.8  & 9  & 0.08 & 6  & 0.33 \\
		& ICVaR-PFT-DPW    & 3  & 5078.4  & 8  & 0.22 & 3  & 0.13 \\
		& POMCPOW          & 7  & 1143.3  & 1  & 0.11 & 1  & 0.01 \\
		& CVaRCost-POMCPOW & 6  & 4151.2  & 3  & 0.42 & 1  & 0.25 \\
		& ICVaR-POMCPOW    & 4  & 2256.5  & 6  & 0.25 & 2  & 0.29 \\
		\bottomrule
	\end{tabular}
\end{table}

\section{Bounding \texorpdfstring{$V^\pi_{M,0}(b_0, \alpha)$}{the Value Function} from Realized State Trajectories}\label{sec:realized_bound}

The main experiments (\Cref{sec:experiments}) estimate per-step CVaR using each planner's operating particle belief. We now ask whether the value function $V^\pi_{M,0}(b_0, \alpha)$ defined by the recursion \eqref{eq:cvar_cost_q_function_def} can be evaluated using only the realized state-dependent trajectories $\{(x_k^{(e)}, a_k^{(e)})\}_{k=0, e=1}^{T-1, N_{\mathrm{ep}}}$ from the $N_{\mathrm{ep}}$ episode rollouts, i.e., \emph{without} any belief reconstruction.

Unrolling \eqref{eq:cvar_cost_q_function_def} from $t = 0$ under policy $\pi$ with terminal condition $V^\pi_{M,T} \equiv 0$:
\begin{equation}\label{eq:value_unrolled}
	V^\pi_{M,0}(b_0, \alpha) \;=\; \mathbb{E}_H\!\left[\sum_{k=0}^{T-1} \gamma^k\, c_\alpha(b_k, a_k)\right], \qquad c_\alpha(b_k, a_k) = \underset{x \sim b_k}{\text{CVaR}_\alpha}[c(x, a_k)],
\end{equation}
where $\mathbb{E}_H$ is expectation over the random state-observation history $H = (x_0, z_1, x_1, \ldots, z_T, x_T)$ generated by $\pi$ from $b_0$ under the POMDP dynamics. The per-step inner CVaR $c_\alpha(b_k, a_k)$ depends on the posterior $b_k = b_k(z_{1:k}, a_{0:k-1})$, and realized state samples from separate episodes cannot reconstruct $b_k$: samples $\{x_k^{(e)}\}_e$ come from distinct observation histories, and the importance reweighting $\tilde{w}_i \propto \prod_n O(z_n \mid x_i^{(n)})$ that defines a belief is absent. Direct evaluation of $V^\pi_{M,0}(b_0, \alpha)$ from realized trajectories alone is therefore impossible. We give a rigorous upper bound via the law of total CVaR.

\begin{lemma}[Law of Total CVaR]\label{lem:total_cvar}
	Let $X$ be a random variable on $(\Omega, \mathcal{A}, P)$ with $\mathbb{E}[|X|] < \infty$, and let $\mathcal{F} \subseteq \mathcal{A}$ be a sub-$\sigma$-algebra. For any $\alpha \in (0, 1]$:
	\begin{equation}\label{eq:total_cvar_bound}
		\text{CVaR}_\alpha(X) \;\geq\; \mathbb{E}\!\left[\text{CVaR}_\alpha(X \mid \mathcal{F})\right].
	\end{equation}
\end{lemma}

\begin{proof}
	Define $f(w, \omega) \triangleq w + \tfrac{1}{\alpha}\mathbb{E}[(X-w)^+ \mid \mathcal{F}](\omega)$ for $w \in \mathbb{R}$ and $\omega \in \Omega$, where $(x)^+ = \max(x, 0)$; this is well-defined a.s.\ since $\mathbb{E}[|X|] < \infty$. By the Rockafellar--Uryasev variational form \citep{rockafellar2000optimization} applied to the conditional distribution of $X$ given $\mathcal{F}$:
	\begin{equation}\label{eq:rv_conditional}
		\text{CVaR}_\alpha(X \mid \mathcal{F})(\omega) \;=\; \inf_{w \in \mathbb{R}} f(w, \omega) \quad \text{a.s.}
	\end{equation}
	The tower property $\mathbb{E}[\mathbb{E}[(X-w)^+ \mid \mathcal{F}]] = \mathbb{E}[(X-w)^+]$ gives $\mathbb{E}[f(w, \cdot)] = w + \tfrac{1}{\alpha}\mathbb{E}[(X-w)^+]$, so applying the variational form to the marginal distribution of $X$:
	\begin{equation}\label{eq:rv_marginal}
		\text{CVaR}_\alpha(X) \;=\; \inf_{w \in \mathbb{R}} \mathbb{E}[f(w, \cdot)].
	\end{equation}
	For any fixed $w_0 \in \mathbb{R}$, the a.s.\ pointwise inequality $\inf_{w \in \mathbb{R}} f(w, \omega) \leq f(w_0, \omega)$ is preserved under expectation:
	\begin{equation}\label{eq:expectation_pointwise}
		\mathbb{E}\!\left[\inf_{w \in \mathbb{R}} f(w, \cdot)\right] \;\leq\; \mathbb{E}[f(w_0, \cdot)].
	\end{equation}
	Taking the infimum over $w_0$ on the right-hand side of \eqref{eq:expectation_pointwise}:
	\begin{equation}\label{eq:expectation_inf_swap}
		\mathbb{E}\!\left[\inf_{w \in \mathbb{R}} f(w, \cdot)\right] \;\leq\; \inf_{w_0 \in \mathbb{R}} \mathbb{E}[f(w_0, \cdot)].
	\end{equation}
	The left-hand side of \eqref{eq:expectation_inf_swap} is $\mathbb{E}[\text{CVaR}_\alpha(X \mid \mathcal{F})]$ by \eqref{eq:rv_conditional}; the right-hand side is $\text{CVaR}_\alpha(X)$ by \eqref{eq:rv_marginal}. This establishes \eqref{eq:total_cvar_bound}.
\end{proof}

\begin{corollary}[Population upper bound on $V^\pi_{M,0}$]\label{cor:realized_upper}
	Let $(x_k, a_k)$ denote the random step-$k$ state-action pair under the rollout distribution of $\pi$ from $b_0$. Then
	\begin{equation}\label{eq:population_upper}
		V^\pi_{M,0}(b_0, \alpha) \;\leq\; \sum_{k=0}^{T-1} \gamma^k\, \text{CVaR}_\alpha\!\left(c(x_k, a_k)\right),
	\end{equation}
	where each $\text{CVaR}_\alpha(c(x_k, a_k))$ is taken over the marginal distribution of $(x_k, a_k)$ under $\pi$.
\end{corollary}

\begin{proof}
	Fix $k \in \{0, \ldots, T-1\}$ and apply \Cref{lem:total_cvar} with $X = c(x_k, a_k)$ and $\mathcal{F} = \sigma(H_k)$, the $\sigma$-algebra of the history $H_k = (b_0, a_0, z_1, \ldots, a_{k-1}, z_k)$. The conditional distribution of $x_k$ given $H_k$ is $b_k(H_k)$ by \Cref{sec:pomdp}, and $a_k = \pi(b_k)$ is $H_k$-measurable, so
	\begin{equation}\label{eq:conditional_identification}
		\text{CVaR}_\alpha\!\left(c(x_k, a_k) \mid \mathcal{F}\right) \;=\; \underset{x \sim b_k}{\text{CVaR}_\alpha}[c(x, a_k)] \;=\; c_\alpha(b_k, a_k) \quad \text{a.s.}
	\end{equation}
	Substituting \eqref{eq:conditional_identification} into \eqref{eq:total_cvar_bound}:
	\begin{equation}\label{eq:per_step_lemma_application}
		\text{CVaR}_\alpha\!\left(c(x_k, a_k)\right) \;\geq\; \mathbb{E}_{H_k}\!\left[c_\alpha(b_k, a_k)\right].
	\end{equation}
	Multiplying \eqref{eq:per_step_lemma_application} by $\gamma^k \geq 0$, summing over $k \in \{0, \ldots, T-1\}$, and applying \eqref{eq:value_unrolled}:
	\begin{equation}
		\begin{aligned}
			\sum_{k=0}^{T-1} \gamma^k\, \text{CVaR}_\alpha\!\left(c(x_k, a_k)\right)
			&\;\geq\; \sum_{k=0}^{T-1} \gamma^k\, \mathbb{E}_{H_k}[c_\alpha(b_k, a_k)] \\
			&\;=\; \mathbb{E}_H\!\left[\sum_{k=0}^{T-1} \gamma^k\, c_\alpha(b_k, a_k)\right]
			\;=\; V^\pi_{M,0}(b_0, \alpha).
		\end{aligned}
	\end{equation}
\end{proof}

A companion lower bound follows from the standard $\text{CVaR}_\alpha(Y) \geq \mathbb{E}[Y]$ applied conditionally step-by-step.

\begin{corollary}[Population lower bound on $V^\pi_{M,0}$]\label{cor:realized_lower}
	Under the setup of \Cref{cor:realized_upper},
	\begin{equation}\label{eq:population_lower}
		V^\pi_{M,0}(b_0, \alpha) \;\geq\; \sum_{k=0}^{T-1} \gamma^k\, \mathbb{E}[c(x_k, a_k)],
	\end{equation}
	where each $\mathbb{E}[c(x_k, a_k)]$ is taken over the marginal distribution of $(x_k, a_k)$ under $\pi$.
\end{corollary}

\begin{proof}
	For any integrable random variable $Y$ and any $w \in \mathbb{R}$, $w + \tfrac{1}{\alpha}\mathbb{E}[(Y-w)^+] \geq w + \mathbb{E}[(Y-w)^+] \geq w + \mathbb{E}[Y - w] = \mathbb{E}[Y]$, using $1/\alpha \geq 1$ and $(y)^+ \geq y$. Taking the infimum over $w$ via the Rockafellar--Uryasev form \citep{rockafellar2000optimization}:
	\begin{equation}\label{eq:cvar_dominates_mean}
		\text{CVaR}_\alpha(Y) \;\geq\; \mathbb{E}[Y].
	\end{equation}
	Apply \eqref{eq:cvar_dominates_mean} conditionally to $c(x, a_k)$ given $b_k$: $c_\alpha(b_k, a_k) = \text{CVaR}_\alpha(c(x, a_k) \mid b_k) \geq \mathbb{E}[c(x, a_k) \mid b_k]$ a.s. Take $\mathbb{E}_{H_k}$ on both sides and use the tower property $\mathbb{E}_{H_k}[\mathbb{E}[c(x, a_k) \mid b_k]] = \mathbb{E}[c(x_k, a_k)]$:
	\begin{equation}\label{eq:per_step_mean_lower_pop}
		\mathbb{E}_{H_k}[c_\alpha(b_k, a_k)] \;\geq\; \mathbb{E}[c(x_k, a_k)].
	\end{equation}
	Multiplying \eqref{eq:per_step_mean_lower_pop} by $\gamma^k \geq 0$, summing over $k \in \{0, \ldots, T-1\}$, and applying \eqref{eq:value_unrolled}:
	\begin{equation}
		V^\pi_{M,0}(b_0, \alpha)
		\;=\; \mathbb{E}_H\!\left[\sum_{k=0}^{T-1} \gamma^k\, c_\alpha(b_k, a_k)\right]
		\;=\; \sum_{k=0}^{T-1} \gamma^k\, \mathbb{E}_{H_k}[c_\alpha(b_k, a_k)]
		\;\geq\; \sum_{k=0}^{T-1} \gamma^k\, \mathbb{E}[c(x_k, a_k)].
	\end{equation}
\end{proof}

Both bounds \eqref{eq:population_upper} and \eqref{eq:population_lower} are population functionals of the marginal step-$k$ distribution of $(x_k, a_k)$ under $\pi$, for which the $N_{\mathrm{ep}}$ realized rollouts provide i.i.d.\ samples at each step $k$. Replacing each marginal functional by its empirical estimator and bounding the estimation error---via DKW and the $1/\alpha$-Lipschitz property of CVaR (\Cref{lem:cvar_cost_sensitivity}) for the upper arm, and via Hoeffding's inequality for the lower arm---yields data-driven finite-sample certificates on $V^\pi_{M,0}(b_0, \alpha)$.

\begin{proposition}[Data-driven upper bound on $V^\pi_{M,0}$]\label{prop:realized_upper_data}
	Assume $c(x, a) \in [c_{\min}, c_{\max}]$ with $\Delta\rho = c_{\max} - c_{\min}$, and let
	\begin{equation}\label{eq:empirical_per_step_cvar}
		\hat{C}_\alpha^{(k)} \;\triangleq\; \hat{C}_\alpha\!\left(\{c(x_k^{(e)}, a_k^{(e)})\}_{e=1}^{N_{\mathrm{ep}}},\, \{1/N_{\mathrm{ep}}\}_{e=1}^{N_{\mathrm{ep}}}\right)
	\end{equation}
	denote the empirical CVaR at step $k$ via \eqref{eq:weighted_cvar_estimator} with uniform weights (equivalently, \eqref{eq:cvar_estimator} with $n = N_{\mathrm{ep}}$). For any $\delta \in (0, 1)$, over the $N_{\mathrm{ep}}$ independent rollouts of $\pi$ from $b_0$:
	\begin{equation}\label{eq:data_driven_upper}
		P\!\left( V^\pi_{M,0}(b_0, \alpha) \;\leq\; \sum_{k=0}^{T-1} \gamma^k\, \hat{C}_\alpha^{(k)} \;+\; \frac{G_T\,\Delta\rho}{\alpha}\,\sqrt{\frac{\ln(2T/\delta)}{2N_{\mathrm{ep}}}} \right) \;\geq\; 1 - \delta,
	\end{equation}
	where $G_T = (1-\gamma^T)/(1-\gamma) = \sum_{k=0}^{T-1} \gamma^k$.
\end{proposition}

\begin{proof}
	Fix $k \in \{0, \ldots, T-1\}$ and $\delta' \in (0, 1)$. Let $F_k(z) \triangleq P(c(x_k, a_k) \leq z)$ be the marginal CDF of the step-$k$ cost under $\pi$, and $F_k^{\mathrm{emp}}(z) \triangleq \tfrac{1}{N_{\mathrm{ep}}}\sum_{e=1}^{N_{\mathrm{ep}}} \mathbf{1}(c(x_k^{(e)}, a_k^{(e)}) \leq z)$ its empirical counterpart. The $N_{\mathrm{ep}}$ realized step-$k$ costs are i.i.d.\ draws from $F_k$ (one per independent rollout of $\pi$ from $b_0$), so by the Dvoretzky--Kiefer--Wolfowitz inequality \citep{massart1990tight}:
	\begin{equation}\label{eq:dkw_step_k}
		P\!\left(\sup_{z \in \mathbb{R}} |F_k^{\mathrm{emp}}(z) - F_k(z)| \;>\; \sqrt{\tfrac{\ln(2/\delta')}{2N_{\mathrm{ep}}}}\right) \;\leq\; \delta'.
	\end{equation}
	By \Cref{lem:cvar_cost_sensitivity} applied to the cost distributions under $F_k$ and $F_k^{\mathrm{emp}}$, on the complement of the DKW failure event:
	\begin{equation}\label{eq:cvar_sensitivity_application}
		\text{CVaR}_\alpha\!\left(c(x_k, a_k)\right) - \hat{C}_\alpha^{(k)} \;\leq\; \frac{\Delta\rho}{\alpha}\,\sup_{z \in \mathbb{R}} |F_k^{\mathrm{emp}}(z) - F_k(z)| \;\leq\; \frac{\Delta\rho}{\alpha}\,\sqrt{\tfrac{\ln(2/\delta')}{2N_{\mathrm{ep}}}}.
	\end{equation}
	Combining \eqref{eq:dkw_step_k} and \eqref{eq:cvar_sensitivity_application}:
	\begin{equation}\label{eq:per_step_cvar_tail}
		P\!\left( \text{CVaR}_\alpha\!\left(c(x_k, a_k)\right) \;>\; \hat{C}_\alpha^{(k)} \;+\; \frac{\Delta\rho}{\alpha}\sqrt{\tfrac{\ln(2/\delta')}{2N_{\mathrm{ep}}}} \right) \;\leq\; \delta'.
	\end{equation}
	Set $\delta' = \delta/T$ and take a union bound over the $T$ events in \eqref{eq:per_step_cvar_tail}, one per step $k \in \{0, \ldots, T-1\}$. With probability at least $1 - \delta$, simultaneously for every $k$:
	\begin{equation}\label{eq:per_step_cvar_uniform}
		\text{CVaR}_\alpha\!\left(c(x_k, a_k)\right) \;\leq\; \hat{C}_\alpha^{(k)} \;+\; \frac{\Delta\rho}{\alpha}\,\sqrt{\tfrac{\ln(2T/\delta)}{2N_{\mathrm{ep}}}},
	\end{equation}
	using $\ln(2/\delta') = \ln(2T/\delta)$. Multiplying \eqref{eq:per_step_cvar_uniform} by $\gamma^k \geq 0$, summing over $k$, and using $\sum_{k=0}^{T-1}\gamma^k = G_T$:
	\begin{equation}\label{eq:marginal_cvar_sum_bound}
		\sum_{k=0}^{T-1} \gamma^k\, \text{CVaR}_\alpha\!\left(c(x_k, a_k)\right) \;\leq\; \sum_{k=0}^{T-1} \gamma^k\, \hat{C}_\alpha^{(k)} \;+\; \frac{G_T\,\Delta\rho}{\alpha}\,\sqrt{\tfrac{\ln(2T/\delta)}{2N_{\mathrm{ep}}}}.
	\end{equation}
	Combining \eqref{eq:marginal_cvar_sum_bound} with \Cref{cor:realized_upper}, which gives $V^\pi_{M,0}(b_0, \alpha) \leq \sum_{k=0}^{T-1} \gamma^k\, \text{CVaR}_\alpha(c(x_k, a_k))$, yields \eqref{eq:data_driven_upper}.
\end{proof}

\begin{proposition}[Data-driven lower bound on $V^\pi_{M,0}$]\label{prop:realized_lower_data}
	Under the assumptions of \Cref{prop:realized_upper_data}, let
	\begin{equation}\label{eq:empirical_per_step_mean}
		\hat{\mathbb{E}}_k \;\triangleq\; \frac{1}{N_{\mathrm{ep}}}\sum_{e=1}^{N_{\mathrm{ep}}} c(x_k^{(e)}, a_k^{(e)})
	\end{equation}
	denote the empirical mean of the step-$k$ realized costs. For any $\delta \in (0, 1)$, over the $N_{\mathrm{ep}}$ independent rollouts of $\pi$ from $b_0$:
	\begin{equation}\label{eq:data_driven_lower}
		P\!\left( V^\pi_{M,0}(b_0, \alpha) \;\geq\; \sum_{k=0}^{T-1} \gamma^k\, \hat{\mathbb{E}}_k \;-\; G_T\,\Delta\rho\,\sqrt{\frac{\ln(T/\delta)}{2N_{\mathrm{ep}}}} \right) \;\geq\; 1 - \delta.
	\end{equation}
\end{proposition}

\begin{proof}
	Fix $k \in \{0, \ldots, T-1\}$ and $\delta' \in (0, 1)$. The realized costs $\{c(x_k^{(e)}, a_k^{(e)})\}_{e=1}^{N_{\mathrm{ep}}}$ are i.i.d.\ across episodes (independent rollouts of $\pi$ from $b_0$) and bounded in $[c_{\min}, c_{\max}]$, so by the one-sided Hoeffding inequality:
	\begin{equation}\label{eq:hoeffding_one_sided}
		P\!\left( \mathbb{E}[c(x_k, a_k)] - \hat{\mathbb{E}}_k \;>\; \Delta\rho\,\sqrt{\tfrac{\ln(1/\delta')}{2N_{\mathrm{ep}}}} \right) \;\leq\; \delta'.
	\end{equation}
	Set $\delta' = \delta/T$ and take a union bound over the $T$ events in \eqref{eq:hoeffding_one_sided}, one per step $k \in \{0, \ldots, T-1\}$. With probability at least $1 - \delta$, simultaneously for every $k$:
	\begin{equation}\label{eq:per_step_mean_uniform}
		\mathbb{E}[c(x_k, a_k)] \;\geq\; \hat{\mathbb{E}}_k \;-\; \Delta\rho\,\sqrt{\tfrac{\ln(T/\delta)}{2N_{\mathrm{ep}}}},
	\end{equation}
	using $\ln(1/\delta') = \ln(T/\delta)$. Multiplying \eqref{eq:per_step_mean_uniform} by $\gamma^k \geq 0$, summing over $k$, and using $\sum_{k=0}^{T-1}\gamma^k = G_T$:
	\begin{equation}\label{eq:marginal_mean_sum_bound}
		\sum_{k=0}^{T-1} \gamma^k\, \mathbb{E}[c(x_k, a_k)] \;\geq\; \sum_{k=0}^{T-1} \gamma^k\, \hat{\mathbb{E}}_k \;-\; G_T\,\Delta\rho\,\sqrt{\tfrac{\ln(T/\delta)}{2N_{\mathrm{ep}}}}.
	\end{equation}
	Combining \eqref{eq:marginal_mean_sum_bound} with \Cref{cor:realized_lower}, which gives $V^\pi_{M,0}(b_0, \alpha) \geq \sum_{k=0}^{T-1} \gamma^k\, \mathbb{E}[c(x_k, a_k)]$, yields \eqref{eq:data_driven_lower}.
\end{proof}

Every term in \eqref{eq:data_driven_upper} and \eqref{eq:data_driven_lower} is a function of the realized state trajectories: $\hat{C}_\alpha^{(k)}$ and $\hat{\mathbb{E}}_k$ are computed directly from the $N_{\mathrm{ep}}$ realized step-$k$ costs, and $(T, \gamma, \alpha, \Delta\rho, N_{\mathrm{ep}}, \delta)$ are fixed. The certificate width scales as $O(G_T\,\Delta\rho/(\alpha\sqrt{N_{\mathrm{ep}}}))$ on the upper arm and $O(G_T\,\Delta\rho/\sqrt{N_{\mathrm{ep}}})$ on the lower arm; the asymmetric $1/\alpha$ factor appears only in the upper bound. The gap between the two data-driven bounds contracts to the population sandwich $\sum_k \gamma^k(\text{CVaR}_\alpha(c(x_k, a_k)) - \mathbb{E}[c(x_k, a_k)])$ as $N_{\mathrm{ep}} \to \infty$, which is zero precisely when conditional distributions $c(x, a_k) \mid b_k$ are a.s.\ point masses (i.e., beliefs are fully informative). Tighter per-step CVaR concentration bounds \citep[e.g.,][]{brown2007large, thomas2019concentration} can be substituted in \eqref{eq:per_step_cvar_tail} without changing the union-bound structure.

%%%%%%%%%%%%%%%%%%%%%%%%%%%%%%%%%%%%%%%%%%%%%%%%%%%%%%%%%%%%%%%%%%%%%%%%%%%%%%%
% NeurIPS Paper Checklist
%%%%%%%%%%%%%%%%%%%%%%%%%%%%%%%%%%%%%%%%%%%%%%%%%%%%%%%%%%%%%%%%%%%%%%%%%%%%%%%

\section*{NeurIPS Paper Checklist}

\begin{enumerate}

\item {\bf Claims}
    \item[] Question: Do the main claims made in the abstract and introduction accurately reflect the paper's contributions and scope?
    \item[] Answer: \answerYes{}
    \item[] Justification: The abstract and introduction state four contributions (CVaR cost formulation, planner compatibility, estimation guarantees, end-to-end bounds), all of which are formally established in the main body and appendix.

\item {\bf Limitations}
    \item[] Question: Does the paper discuss the limitations of the work performed by the authors?
    \item[] Answer: \answerYes{}
    \item[] Justification: Section~7 (Conclusion) includes a ``Limitations and future work'' paragraph discussing the per-step vs.\ trajectory-level risk tradeoff and open extensions.

\item {\bf Theory assumptions and proofs}
    \item[] Question: For each theoretical result, does the paper provide the full set of assumptions and a complete proof?
    \item[] Answer: \answerYes{}
    \item[] Justification: All assumptions (bounded costs $c_{\min} \leq c(x,a) \leq c_{\max}$, conditions (i)--(vi) of \citet{lim2023optimality} including the bounded R\'{e}nyi divergence $d_\infty(b_\ell \| q_\ell) \leq d_\infty^{\max}$, and the distributional discrepancy condition $\varepsilon(\ell, N_p) < \alpha$) are stated explicitly. Complete proofs of all theorems, propositions, and lemmas are provided in the appendix.

\item {\bf Experimental result reproducibility}
    \item[] Question: Does the paper fully disclose all the information needed to reproduce the main experimental results of the paper to the extent that it affects the main claims and/or conclusions of the paper?
    \item[] Answer: \answerYes{}
    \item[] Justification: \Cref{sec:experiments} reports all experimental parameters (number of particles, planning timeout, discount factor, risk level, number of episodes), and \Cref{sec:hyperparameter_tuning} gives the per-(environment, planner) tuned hyperparameters.

\item {\bf Open access to data and code}
    \item[] Question: Does the paper provide open access to the data and code?
    \item[] Answer: \answerYes{}

\item {\bf Experimental setting/details}
    \item[] Question: Does the paper specify all the training and test details?
    \item[] Answer: \answerYes{}
    \item[] Justification: \Cref{sec:experiments} specifies environment parameters, planner configurations, number of episodes, and evaluation metric; \Cref{sec:hyperparameter_tuning} reports the tuning protocol and final hyperparameter values.

\item {\bf Experiment statistical significance}
    \item[] Question: Does the paper report error bars or confidence intervals?
    \item[] Answer: \answerYes{}
    \item[] Justification: Table~1 reports 95\% confidence intervals for all metrics across 500 episodes.

\item {\bf Experiments compute resources}
    \item[] Question: Does the paper report the computational resources used?
    \item[] Answer: \answerNo{We use high performance computing clusters that do not specify their exact infra.}

\item {\bf Code of ethics}
    \item[] Question: Does the research conducted in the paper conform to the NeurIPS Code of Ethics?
    \item[] Answer: \answerYes{}

\item {\bf Broader impacts}
    \item[] Question: Does the paper discuss both potential positive societal impacts and potential negative societal impacts?
    \item[] Answer: \answerYes{}
    \item[] Justification: A Broader Impact Statement is included at the end of the main body.

\item {\bf Safeguards}
    \item[] Question: Does the paper describe safeguards that have been put in place for responsible release of data or code?
    \item[] Answer: \answerNA{}

\item {\bf Licenses for existing assets}
    \item[] Question: Are the creators or original owners of assets used in the paper properly credited?
    \item[] Answer: \answerYes{}
    \item[] Justification: The environments and baseline algorithms are properly cited.

\item {\bf New assets}
    \item[] Question: Are new assets introduced in the paper well documented and is the documentation provided alongside the assets?
    \item[] Answer: \answerNA{}

\item {\bf Crowdsourcing and human subjects}
    \item[] Question: For crowdsourcing experiments and research with human subjects, does the paper include the full text of instructions given to participants?
    \item[] Answer: \answerNA{}

\item {\bf IRB approvals}
    \item[] Question: Does the paper describe potential risks incurred by study participants?
    \item[] Answer: \answerNA{}

\end{enumerate}

\end{document}